\documentclass{article}

\usepackage{microtype}
\usepackage{graphicx}
\usepackage{subfigure}
\usepackage{booktabs} 

\usepackage{amsmath}
\usepackage{amsthm}
\usepackage{bm}
\usepackage{amssymb}
\usepackage{cancel}

\usepackage{hyperref}

\usepackage[nonotice]{icml2018}

\newtheorem{theorem}{Theorem}[section]

\begin{document}

\renewcommand{\O}{\mathcal{O}}
\newcommand{\bh}{\boldsymbol{h}}
\newcommand{\bo}{\boldsymbol{o}}
\newcommand{\bp}{\boldsymbol{p}}
\newcommand{\bq}{\boldsymbol{q}}
\newcommand{\bs}{\boldsymbol{s}}
\newcommand{\bw}{\boldsymbol{w}}
\newcommand{\bx}{\boldsymbol{x}}
\newcommand{\by}{\boldsymbol{y}}
\newcommand{\bz}{\boldsymbol{z}}
\newcommand{\bphi}{\boldsymbol{\phi}}
\newcommand{\btheta}{\boldsymbol{\theta}}
\renewcommand{\vec}{\mathrm{\mathop{vec}}}

\twocolumn[
\icmltitle{Adaptive Sampled Softmax with Kernel Based Sampling}

\begin{icmlauthorlist}
\icmlauthor{Guy Blanc}{sta}
\icmlauthor{Steffen Rendle}{goo}
\end{icmlauthorlist}

\icmlaffiliation{sta}{Work done during internship at Google, Mountain View, USA}
\icmlaffiliation{goo}{Google, Mountain View, USA}

\icmlcorrespondingauthor{Guy Blanc}{guy.blanc@gmail.com}
\icmlcorrespondingauthor{Steffen Rendle}{srendle@google.com}

\icmlkeywords{softmax,sampled softmax,extreme classification}

\vskip 0.3in
]

\printAffiliationsAndNotice{}

\begin{abstract}
Softmax is the most commonly used output function for multiclass problems and is widely used in areas such as vision, natural language processing, and recommendation.
A softmax model has linear costs in the number of classes which makes it too expensive for many real-world problems.
A common approach to speed up training involves sampling only some of the classes at each training step.
It is known that this method is biased and that the bias increases the more the sampling distribution deviates from the output distribution.
Nevertheless, almost all recent work uses simple sampling distributions that require a large sample size to mitigate the bias.
In this work, we propose a new class of kernel based sampling methods and develop an efficient sampling algorithm.
Kernel based sampling adapts to the model as it is trained, thus resulting in low bias.
It can also be easily applied to many models because it relies only on the model's last hidden layer.
We empirically study the trade-off of bias, sampling distribution and sample size and show that kernel based sampling results in low bias with few samples.
\end{abstract}

\section{Introduction}
Classification problems with a large number of classes are common in many language tasks~\cite{mikolov:nips13,bengio:08} and recommender systems~\cite{covington:16}.
A standard and effective approach to these classification tasks is to use some model, such as a neural network, to compute a logit for each class, and assume that the class probabilities are a softmax of the logits.
Computing class probabilities with softmax involves a normalization step where a \emph{partition function} over the logits of all classes is computed.
For learning the model parameters, an optimization algorithm, e.g.,  stochastic gradient descent, needs to compute the gradients with respect to the loss.
When the number of classes, $n$, is large, computing the probability of each class is often too slow, as the time for each training step grows linearly with $n$.
Sampled softmax, which creates a sample of $m<n$ classes in every update step, is commonly used when the number of classes becomes too large.
It is well known that sampled softmax is biased \cite{bengio:08}, i.e., it does not converge to the same loss as a full softmax -- no matter how many update steps are taken.
The only way to eliminate the bias is to sample from the softmax distribution which is not efficient.
For any other sampling distribution, there are two directions to mitigate the bias: (i)~choose a sampling distribution that is closer to softmax, or (ii)~increase the sample size, $m$ -- which is trivial but costly.
Early work \cite{bengio:08} has shown that a good sampling distribution should be adaptive and should depend on the model's output.

While the importance of the sampling distribution is known, surprisingly, almost all recent applications use simple sampling distributions such as uniform or global popularity, which require large sample sizes to achieve an acceptable bias.
One reason for this trend could be that the models have tended to get more complex, e.g. stacked LSTMs, very deep networks, convolutional NN, etc. which makes it hard to design an efficient sampling distribution that adapts to the model.

In this work, we propose a new class of sampling distributions that approximate softmax but are efficient to compute.
The proposed sampling distributions are defined over the model's output, making them adaptive to the input, the model's structure, and the current model parameters.
The main idea is to sample proportionally to a non-negative kernel.
We show that kernels allow us to compute the partition function efficiently in the kernel space.
This result can be used in a divide and conquer algorithm that samples in $\O(D\,\log\,n)$ time, where $D$ is the dimension of the kernel space.
We suggest the quadratic kernel as an approximation for (absolute) softmax.
See Section~\ref{sec:quadratic_kernel} for details.
Kernel based sampling is generic and can be applied directly to any model where the final layer is a dot product between a hidden layer and class embeddings. 

We study the bias of uniform, quadratic kernel and softmax sampling empirically and show that the quadratic kernel needs one to two orders of magnitude less samples than uniform to reach the same quality as full softmax.
A second observation is that once the bias is eliminated, more samples usually do not increase the convergence speed.

\section{Modelling Large Multiclass Problems}

In this section, we first formalize the multiclass softmax and then recap its sampling version.

Let $\by \in [0,1]^n$ with $\sum_{i=1}^n y_i = 1$ be a distribution over $n$ classes for an input $\bx \in \mathcal{X}$.
The goal of supervised learning is to find a function that explains a set of observed pairs $(\bx, \by)$ of input $\bx$ and label $\by$.
Let $\bo : \mathcal{X} \times \Theta \rightarrow \mathbb{R}^n$ be such a function that maps an input $\bx$ to a raw score for each class.
The model function $\bo$ is parameterized by model parameters $\btheta \in \Theta$.
To shorten notation, we drop the arguments $\bx$ and $\btheta$ from $\bo$ and all derived functions, whenever the dependency is clear.

\subsection{Full Softmax}

A \emph{softmax} model links the model outputs $\bo$ to a class probability distribution $\bp \in [0,1]^n$ with $\sum_{i=1}^n p_i = 1$ by applying an exponential function
\begin{align}
    \label{softmax probability}
    p_i := \frac{\exp(o_i)}{\sum_{j=1}^n\exp(o_j)} 
\end{align}
The denominator of $p_i$ is also known as the \emph{partition function} and takes at least $\O(n)$ time to compute.
For softmax, the output $\bo$ is often referred to as the \emph{logits}.
The loss $L$ of a parameter setting $\btheta$ is measured by the cross entropy between $\by$ and $\bp$
\begin{align*}
    L(\by, \bp) := -\sum_{i=1}^n y_i \log p_i = \log \sum_{i=1}^n \exp(o_i) - \sum_{i=1}^n y_i\, o_i
\end{align*}
This \emph{full} softmax loss depends on \emph{all} classes.
Thus, learning a full softmax is expensive when the number of classes, $n$, is large.

\subsection{Sampled Softmax}

Sampled softmax aims to approximate a full softmax during model training \cite{bengio:08,bengio:03}.
Rather than computing the loss over all classes, only the positive class and a sample of $m$ negative classes are considered.
Each negative class is sampled with probability $q_i$ with replacement.
For the rest of the paper, we assume w.l.o.g. that there is one positive class per training example, i.e., $\by \in \{0,1\}^n$.
The vector $\bs \in \{1,\ldots,n\}^{m+1}$ represents a sample of classes and stores the index of the positive and the index of the $m$ sampled negative classes.
For instance, $\bs = (2,6,7,6,3)$, represents a sample of size $m=4$ with the positive class at index $2$ and four negative classes, where the class at index $6$ was sampled twice, and the classes at index $7$ and $3$ once each. 

Just as $\bo$, $\by$, and $\bp$ with cardinality $n$ refer to important characteristics of all the classes, $\bo'$, $\by'$, and $\bp'$ with cardinality $m+1$ reflect similar values for a sample $\bs$.
First, each index $i \in \{1,\ldots, m+1\}$ of the sample $\bs$ is assigned an adjusted logit $o'_i$.
\begin{align}
    \label{softmax correction}
    o'_i := \left\{
    \begin{array}{ll}
      o_{s_i} - \ln(m\, q_{s_i}) & \text{if } y_{s_i} = 0 \\
      o_{s_i} - \ln(1) = o_{s_i} & \text{else}
    \end{array}
  \right.
\end{align}
The adjusted logit corrects the true logit $o_{s_i}$ by the expected number of occurences of a class $s_i$ in the sample $\bs$.
This correction ensures that in the limit of $m \to \infty$, sampled softmax is unbiased \cite{bengio:08}.

Second, $\bp'$ is the softmax probability distribution computed over adjusted logits $\bo'$, and $\by'$ is a projection of the original labels $\by$ to the sample $\bs$.
\begin{align}
    \label{softmax sample probability}
    p'_i := \frac{\exp(o'_{i})}{\sum_{j=1}^{m+1}\exp(o'_{j})},\quad
    y'_i := y_{s_i}
\end{align}
The loss of a sample $\bs$ is the cross entropy $L(\by', \bp')$ between predicted probabilities $\bp'$ and labels $\by'$.
In contrast to full softmax, the loss of sampled softmax depends only on (at most) $m+1$ different classes.

\subsection{Importance of the Sampling Distribution}

\label{sec:unbiased}
Sampled softmax can be viewed as an algorithm that generates an estimator for the full softmax gradient with respect to the logits.
The full softmax gradient with respect to a logit $o_i$ is
\begin{align}
    \label{softmax gradient}
    \frac{\partial L(\bp,\by)}{\partial o_i} =  p_i - y_i
\end{align}
whereas the sampled softmax gradient with respect to an original logit $o_i$ reads
\begin{align}
    \label{sample softmax gradient}
    \frac{\partial L(\bp',\by')}{\partial o_i} &= \sum_ {j = 1}^{m+1} I(s_j = i) (p'_j - y'_j) \\
    &=  \sum_{j=1}^{m+1} I(s_j = i) p'_j - y_i \notag
\end{align}
Ideally, we would like to pick a sampling distribution such that sampled softmax converges to the same value as full softmax.
At the very least, we would like to guarantee convergence with infinitely small step size and infinitely many steps.
That is guaranteed if the sampled softmax estimator is unbiased:
\begin{align}
    \label{eq:exgradient}
     &E\left[\frac{\partial L(\bp',\by')}{\partial o_i}\right] \stackrel{?}{=} \frac{\partial L(\bp,\by)}{\partial o_i} \\
     \Leftrightarrow &E\left[\sum_{j=1}^{m+1} I(s_j = i) p'_j\right] \stackrel{?}{=} p_i
\end{align}
\citet{bengio:08} have shown that sampling proportional to the softmax probability, $q_i = p_i \propto \exp(o_i)$, is an unbiased estimator.
In fact, $q_i = p_i \propto \exp(o_i)$ is the \emph{only} unbiased estimator.
\begin{theorem}
\label{th:unbiased}
  The gradient of sample softmax is an unbiased estimator of the full softmax gradient iff $q_i = p_i \propto \exp(o_i)$.
\end{theorem}
We include a detailed proof in the appendix.

\subsection{Properties of a Good Sampling Distribution}
\label{sec:properties_of_good_distribution}

The last sections argued that sampled softmax is biased and the only way to mitigate the bias are (1)~choose a sampling distribution $q_i$ closer to softmax $p_i$ or (2)~increase the sample size, $m$.
The closer the sampling distribution $q_i$ reflects $p_i$, the smaller the sampling size that is needed for low bias.
Finally, we highlight three properties of the softmax distribution, $p_i \propto o_i(\bx,\theta)$, that a good sampling distribution, $\bq$, should meet as well.
\begin{enumerate}
        \item \emph{Example dependent}: Every input, $\bx$, has an individual sampling distribution, because the output, $\bo(\bx)$, depends on the input, $\bx$.
        \item \emph{Model structure dependent}: The sampling depends on the functional structure of $\bo$.
        For instance, if $\bo$ is an LSTM, the sampling distribution should not be represented by simple bigrams.
        \item \emph{Model parameter dependent}: The sampling distribution changes while the model is learned, because $\bo$ depends on the model parameters.
\end{enumerate}
Common sampling schemes such as uniform or popularity based sampling are neither example nor model dependent.
In the following section, we introduce a sampling algorithm that meets these criteria and is efficient.

\section{Kernel Based Sampling}

Sampling directly from $q_i \propto \exp(o_i)$ requires computing the partition function and is as expensive as computing the full softmax.
The motivation for sampling is to avoid that inefficiency, so sampling from $q_i \propto \exp(o_i)$ is not a good option.
In this section, we propose efficient sampling distributions that depend on the example $\bx$, the model structure $\bo$ and the model parameter $\theta$ as highlighted in Section~\ref{sec:properties_of_good_distribution}.

So far, we have ignored how the logits $\bo$ are computed.
In the following, we assume that $o_i$ is a dot product between a context or query embedding, $\bh \in \mathbb{R}^d$, and a class embedding, $\bw_i \in \mathbb{R}^d$.
This type of model is extremely common with many examples such as deep neural networks and factorization models.
For example, $\bh$ could be the last hidden layer of a deep neural network and $\boldsymbol{W} \in \mathbb{R}^{n \times d}$ the last matrix of weights, such that $\bo = \boldsymbol{W}^T\bh$.
The cost of computing the full softmax on a dot product model is $\O(nd)$.

\subsection{Kernel Based Distributions}
\label{supported distributions}

We consider sampling distributions that are proportional to some function $K : \mathbb{R}^d \times \mathbb{R}^d \to \mathbb{R}^+$.
We assume that $K$ is a kernel function for a $D$ dimensional space, i.e., there exists a mapping $\bphi: \mathbb{R}^d \to \mathbb{R}^D$ such that $K(\boldsymbol{a},\boldsymbol{b}) = \langle \bphi(\boldsymbol{a}), \bphi(\boldsymbol{b}) \rangle$.
Thus, the sampling distribution can be written as:
\begin{align}
     q_i= \frac{K(\bh, \bw_i)}{\sum_{j=1}^n K(\bh, \bw_j)}
        = \frac{K(\bh, \bw_i)}{\left\langle \bphi(\bh), \underbrace{\sum_{j=1}^n \bphi(\bw_j)}_{=:\bz \in \mathbb{R}^D} \right\rangle} 
    \label{eq:kernel_distribution}
\end{align}
The last step shows the key property that we gain from a kernel: the summation over all classes can be isolated from the query $\bh$ -- i.e., the partition function becomes a simple dot product between a query vector and a summary vector $\bz$.
This summary vector is independent of the query and can be precomputed.

\subsection{Sampling with Divide and Conquer}
\label{sampling algorithm}

\begin{figure*}[ht]
    \centering
    \subfigure[sampling a class]{\label{fig:sampling_path}\includegraphics[height=0.57\columnwidth]{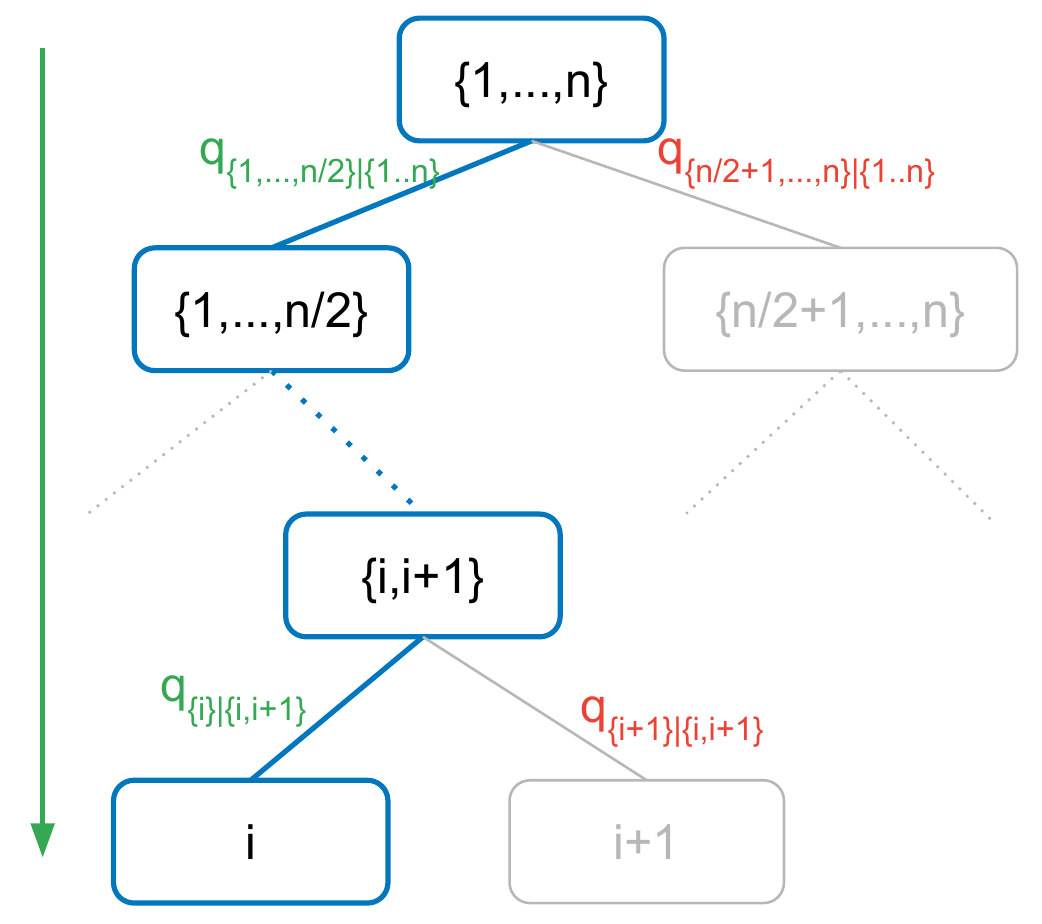}}
    \subfigure[updating statistics]{\label{fig:update}\includegraphics[height=0.57\columnwidth]{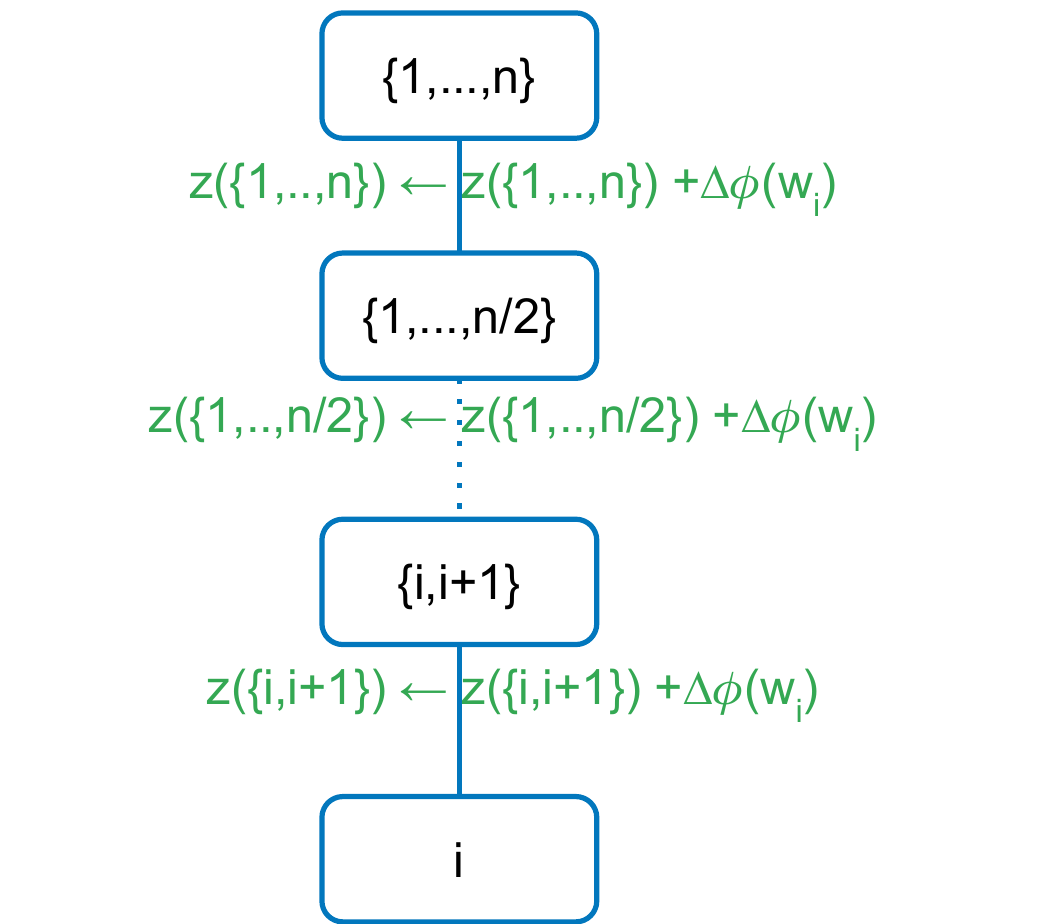}}
    \subfigure[large branching factor for leaves]{\label{fig:sampling_path_large}\includegraphics[height=0.57\columnwidth]{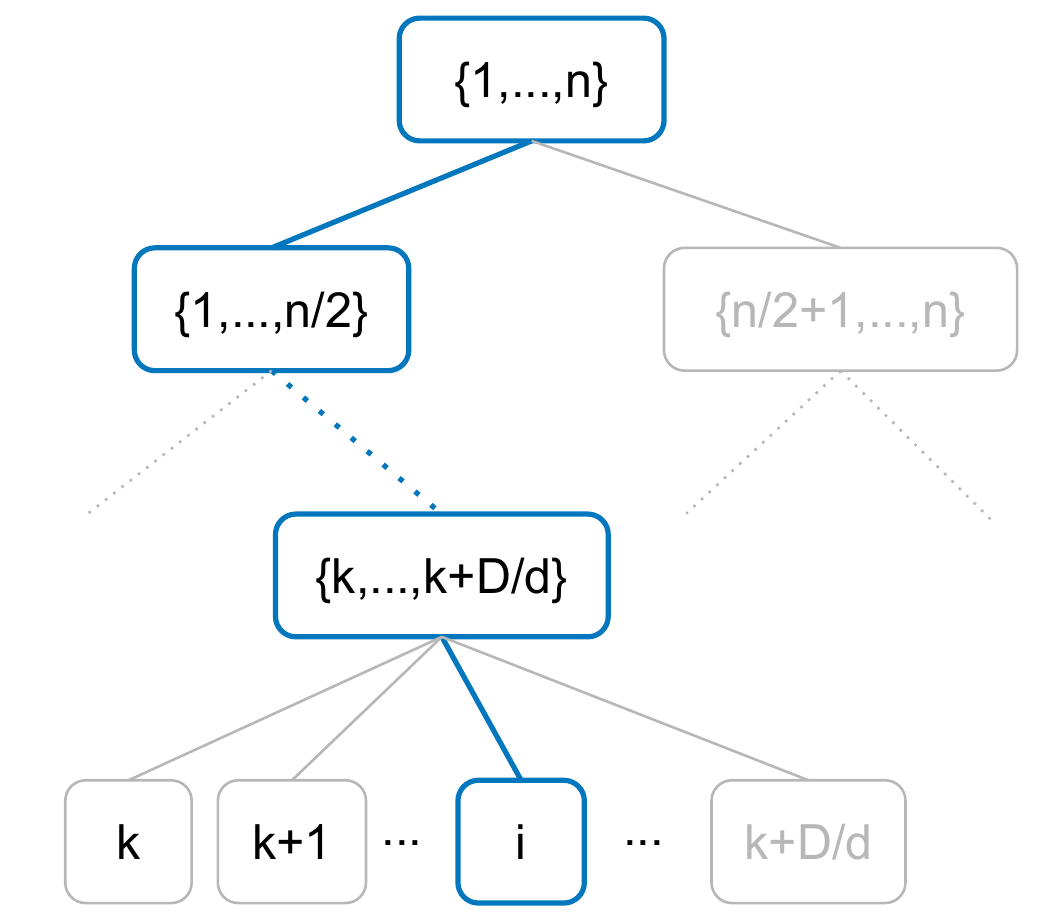}}
     \caption{Divide and conquer algorithm for sampling from a kernel distribution $q$.
     Figure \ref{fig:sampling_path} shows how to sample subsets starting from all classes $\{1,\ldots,n\}$ until a single item $i$ is reached.
     After the class embedding, $\bw_i$, of class $i$ changes from $\bw^{\text{old}}_i$ to $\bw^{\text{new}}_i$, all statistics, $\bz$, on the sampling path of $i$ are updated by $\Delta\phi(\bw_i) := \phi(\bw^{\text{new}}_i) -\phi(\bw^{\text{old}}_i)$ (Figure \ref{fig:update}).
     To minimize storage costs for statistics $\bz$, it is beneficial to use a higher branching factor of $\O(\frac{D}{d})$ for the leaves (Figure \ref{fig:sampling_path_large}).}
    \label{fig:kernel_sampling}
\end{figure*}

The kernel gives the ability to compute the probability of \emph{one} class efficiently.
Next, we discuss how this property can be used for efficient sampling from \emph{all} classes.
Instead of sampling a class directly from all the possible classes, we sample a subset of classes recursively until the subset has only one class (see Figure~\ref{fig:sampling_path}).
To formalize this algorithm, we introduce $C \subseteq \{1, \ldots,n\}$ as a set of classes and define $\bz(C) := \sum_{j \in C} \bphi(\bw_j)$.
Let $C' \cup C'' = C$ be a partition of $C$ into two disjoint sets $C'$ and $C'' = C \setminus C'$.
We define the probability, $q_{C'|C}$, of sampling the set $C'$ from $C$, as the sum of the probabilities of its elements:
\begin{align}
    q_{C'|C} &:= \sum_{j \in C'} \frac{K(\bh, \bw_j)}{\sum_{l \in C}  K(\bh,\bw_l)} \\
             &= \frac{\langle \phi(\bh), \sum_{j \in C'} \phi(\bw_j)\rangle}{\langle \phi(\bh), \sum_{l \in C} \phi(\bw_l) \rangle} 
             = \frac{\langle \phi(\bh), \bz(C')\rangle}{\langle \phi(\bh), \bz(C) \rangle} \notag 
\end{align}
If we know $\bz$ for $C$ and $C'$, we can sample from this distribution in $\O(D)$ time.
This scheme can be applied recursively to the sampled subset until the subset contains exactly one class.
With $n$ classes and two sets of equal size at each step, this takes $\log_2 n$ steps and in total the time for sampling a class proportional to $\bq$ is $\O(D\,\log_2 n)$.

\subsubsection{Analysis}
\paragraph{Correctness}
The correctness of the divide and conquer algorithm, i.e., that it samples proportional to the kernel distribution (eq.~\ref{eq:kernel_distribution}), is easy to show.
Assume the algorithm samples class $i$ and the intermediate sets were $C_1$, $C_2$, \ldots, $C_{\log n - 1}$.
The probability for sampling class $i$ with the divide and conquer algorithm is equal to $q_i$:
\begin{align*}
     & q_{C_1 | \{1,\ldots,n\}} \, q_{C_2|C_1} \, \ldots \, q_{\{i\}|C_{\log n - 1}} \\
    =&  \frac{\cancel{\langle \phi(\bh), \bz(C_1)\rangle}}{\langle \phi(\bh), \bz(\{1,\ldots,n\}) \rangle}
        \frac{\cancel{\langle \phi(\bh), \bz(C_2)\rangle}}{\cancel{\langle \phi(\bh), \bz(C_1) \rangle}}
        \ldots
        \frac{\langle \phi(\bh), \phi(\bw_i) \rangle}{\cancel{\langle \phi(\bh), \bz_{\log n - 1} \rangle}} \\
    =& \frac{K(\bh, \bw_i)}{\langle \phi(\bh), \bz \rangle}
    = q_i \qed
\end{align*}

\paragraph{Runtime}
The divide and conquer algorithm assumes that $\bz(C)$ is known for every set that is involved in sampling.
As sampling is independent of the particular choice of the splits, we can choose any arbitrary (binary and balanced) split and keep it fixed.
In total, there are $n$ many sets that are arranged in a tree like structure and each class appears in exactly $\log_2 n$ many sets.
This allows to precompute $\bz(C)$ for any of the $n$ sets.
If we update an embedding, $\bw_i$ during training, we can also update all sets in which $i$ appears in, in time $O(D \log n)$ by updating $\bz(C)$ for every node along the path from the root to that embedding.
Figure~\ref{fig:update} illustrates the update process.

\subsubsection{Practical Considerations}

\label{practical considerations}
\paragraph{Less Memory}
The structure described so far has $\O(n)$ nodes in total, each of which must store $\O(D)$ information for $\bz$.
This means $\O(nD)$ space is required to store it.
Here we will describe how to reduce that to $\O(nd)$ space while maintaining fast sampling and updating.

Instead of splitting sets until they reach the trivial size $1$, we suggest to stop splitting as soon as the size of a set is $\O(\frac{D}{d})$.
This leads to the tree having a total of $\O(\frac{nd}{D})$ sets, and requires $O(nd)$ memory.
Figure \ref{fig:sampling_path_large} sketches the sampling process with a larger branching factor for the leaves.
Increasing the branching factor seems very costly for the final step because the algorithm has to sample from a set of $\O(\frac{D}{d})$ many classes.
However, for most kernels, $K(\boldsymbol{a},\boldsymbol{b})$ can be computed efficiently in $\O(d)$ time, e.g., for kernels of the form $K(\boldsymbol{a},\boldsymbol{b}) = f(\langle \boldsymbol{a}, \boldsymbol{b} \rangle)$.
Thus, performing the last step in the original space takes $\O(d\frac{D}{d})=\O(D)$ time even with a naive implementation.
The proposed modification decreases the height of the tree from $\O(\log_2 n)$ to $\O(\log_2 \frac{nd}{D})$, and adds a final step to sampling with time $\O(D)$.
The total sampling time is thus $\O(D(1 + \log_2 \frac{nd}{D})$ which is still $\O(D \log_2 n)$.

\paragraph{Multiple Partial Samples}

Usually, we want to sample several negatives from $\bq$.
Instead of applying the divide and conquer algorithm $m$ times, a single run could return all the $\frac{D}{d}$ leaf nodes.
This would require an additional correction in sampled softmax to accept a weight on each sample.
Then, instead of $q_i$ being the probability of sampling a particular class, it is the probability of sampling a class multiplied by the weight given to that class when it is sampled.
The drawback of this approach is that the samples are not independent and likely more total samples would be needed.
We do not further investigate this approach, but in some applications, faster sampling might justify the cost of requiring a few more samples.

\subsection{Quadratic Kernel}
\label{sec:quadratic_kernel}

One obvious choice for a kernel is a quadratic function $K(\bh, \bw_i) = \alpha \langle \bh, \bw_i \rangle^2 + 1$.
This function is conveniently always positive.
Its feature representation is
\begin{align}
\phi(\boldsymbol{a}) = \left[\sqrt{\alpha}\,\vec(\boldsymbol{a} \otimes \boldsymbol{a}), 1\right] 
\end{align}
with $D = O(d^2)$, allowing for $O(d^2 \log n)$ sampling.
It is also a reasonably good approximation of $\exp$ near the origin, where many logits tend to be.
However, a quadratic function is a poor approximation for negative logits and would oversample classes with negative logits.
To align the sampling distribution $\bq$ better with the prediction distribution $\bp$, we suggest a modification of the softmax probability, $\bp$, in eq.~(\ref{softmax probability}) to an \emph{absolute softmax}
\begin{align}
    \label{softmax abs probability}
    p_i = \frac{\exp(|o_i|)}{\sum_{j=1}^n\exp(|o_j|)} 
\end{align}
This modified prediction distribution does not negatively impact the expressiveness because softmax is shift invariant, i.e., $q_i \propto \exp(o_i) \propto \exp(o_i)\,\exp(c) = \exp(o_i + c)$ for any constant $c \in \mathbb{R}$.
In particular, any \emph{softmax} solution has a corresponding \emph{absolute softmax} solution by shifting the logits, $\bo$, of the softmax solution by any $c$ large enough to make all the logits nonnegative.
We investigated also empirically the quality of \emph{softmax} and \emph{absolute softmax} as prediction distribution when learning without sampling, i.e., full softmax, and both performed very similarly\footnote{Similar empirical findings were obtained by \citet{spherical:15} on various tasks.} on the datasets of Section~\ref{sec:datasets}.
Finally, analogous to Section~\ref{sec:unbiased}, for absolute softmax as the prediction distribution, the only unbiased sampling distribution is absolute softmax.
This follows directly from Theorem~\ref{th:unbiased} in the appendix, because the analysis was shown for $p_i \propto(o_i)$ and \emph{any} output $o_i$, so it also holds for the modified output $|o_i|$.
Therefore, we suggest to use an absolute softmax as prediction distribution when sampling from a symmetric kernel like the quadratic kernel and a standard softmax in other cases.

Another way to look at absolute softmax is to add an additional layer to $\bo$ that performs $|\bo|$ and then passing the result to a standard softmax.

\section{Experiments}

In this section, we empirically investigate the trade-off between bias, sampling distribution, and number of samples. 

\subsection{Experimental Setup}

\begin{figure*}[ht]
    \includegraphics[width=0.66 \columnwidth]{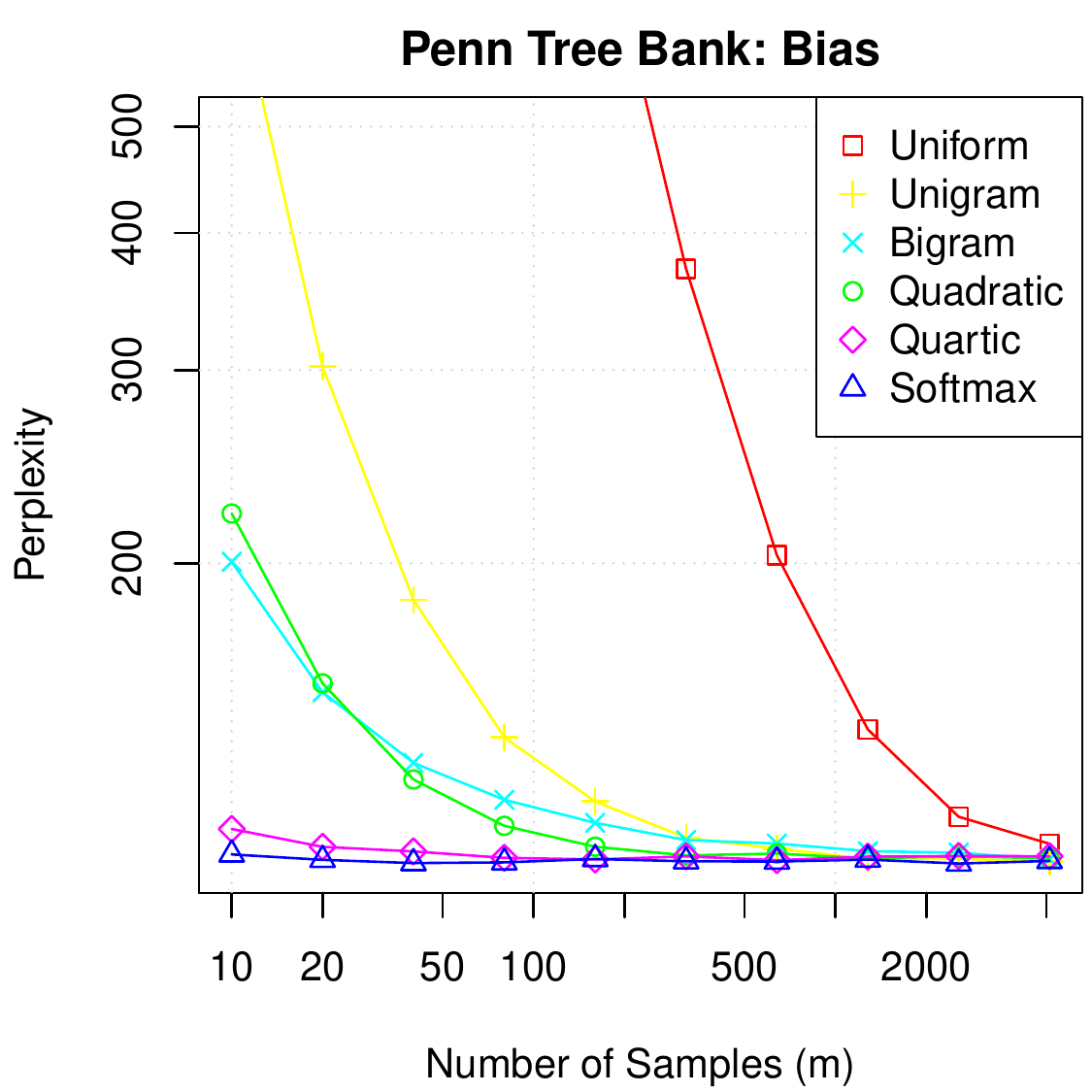}
    \includegraphics[width=0.66\columnwidth]{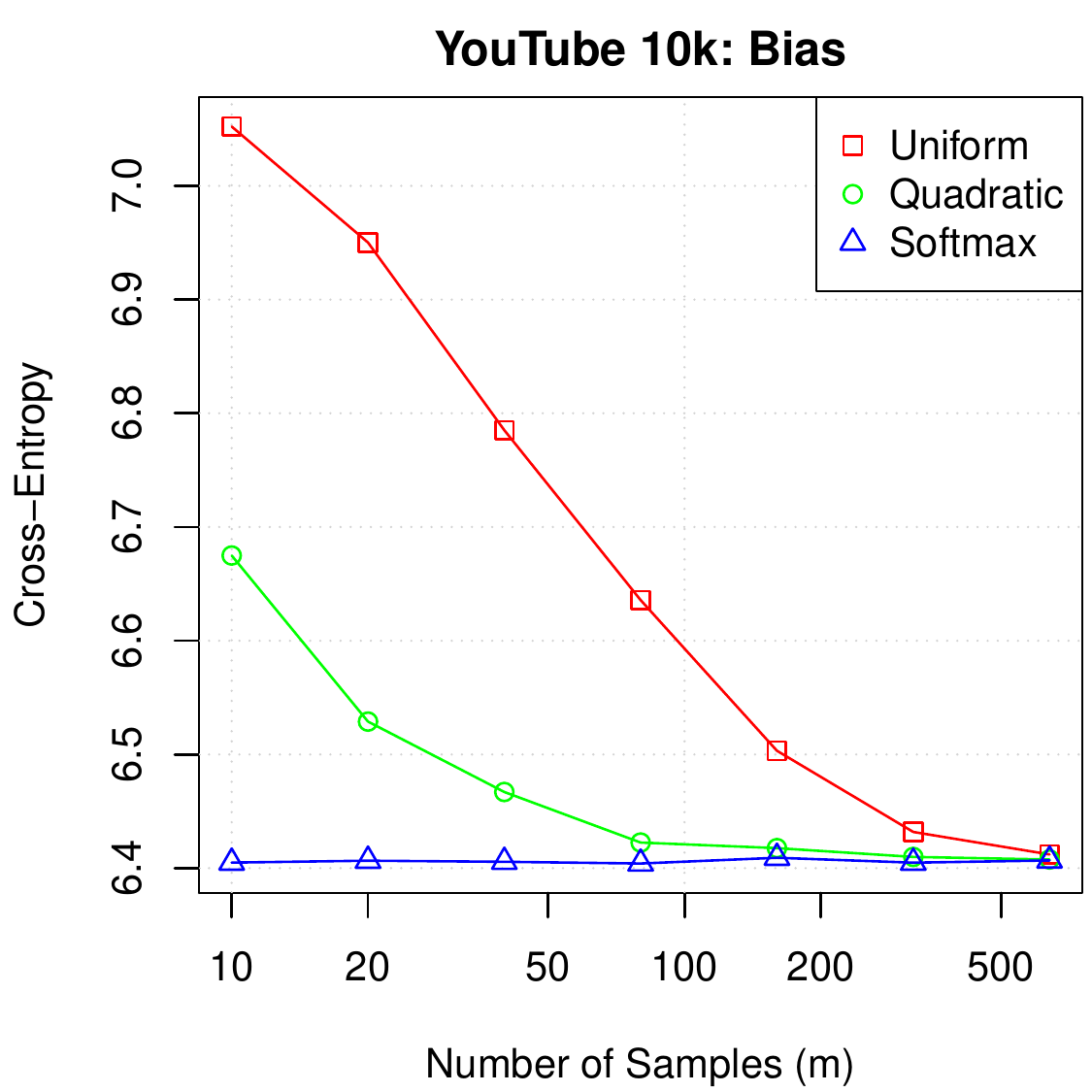}
    \includegraphics[width=0.66\columnwidth]{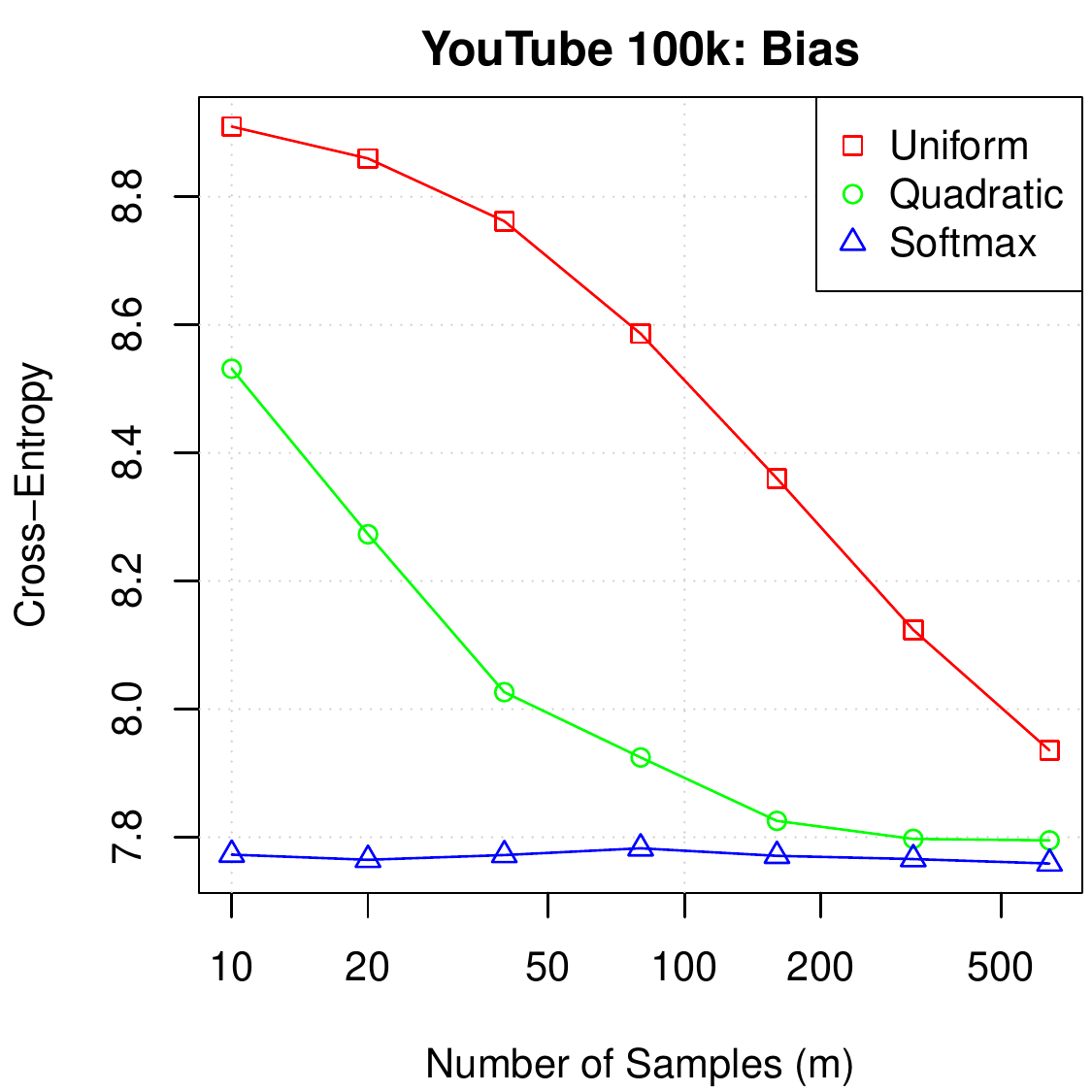}
	\caption[bias]{Final model quality when training a sampled softmax with different sampling distributions (\emph{uniform}, \emph{quadratic}, \emph{softmax}) and number of samples, $m$.
    The quadratic distribution needs one to two orders of magnitude less samples than uniform sampling to learn a low bias model.
	Penn Tree Bank includes additional results for a \emph{unigram} and a \emph{bigram} sampler which are common sampling distributions in NLP sequence tasks.
    The results for Penn Tree Bank also include a \emph{quartic} sampler which is a 4-th degree polynomial kernel with $q_i \propto o_i^4 + 1$.}
    \label{Bias Plots}
\end{figure*}

\subsubsection{Datasets and Models}
\label{sec:datasets}
We study sampled softmax on a natural language processing (NLP) problem and a recommender system dataset.

\paragraph{Penn Tree Bank}

For the NLP problem, we learn a language model on the Penn Tree Bank dataset~\cite{ptb}, a dataset with approximately 1 million training words and a vocabulary of size 10,000.
We use the well-studied "medium regularized LSTM" implementation\footnote{\url{https://www.tensorflow.org/tutorials/recurrent}} of \citet{zaremba:14}.
We made one minor modification, and changed the units per layer from 650 to 200.
Doing so ensures that the expressiveness of the model is small enough that we do not need to worry about early-stopping, and dropout on its own is a sufficient regularizer.
We report the perplexity loss as in \cite{zaremba:14}.

\paragraph{YouTube}
In this recommendation dataset, we predict which video a user will watch next based upon various user features and the three previously watched videos.
We train a deep neural network where the user features and previous videos are the input and the output is the watch probability over all videos.
To study the effect on sampling, we created two versions of the dataset: YouTube10k, and YouTube100k with 10,000, and 100,000 videos (=classes) respectively.
The 10k dataset has about 113 million training examples, and the 100k dataset about 187 million examples. 
For recommender systems, a common evaluation protocol is to rank videos by their scores and then use some ranking metric (e.g. mean average precision) to measure the quality of a model.
Here, we only wish to measure how well sampled softmax approximates a full softmax.
Thus, we measure the cross-entropy loss of full softmax.
In our YouTube experiments, the cross-entropy loss was also highly correlated with ranking metrics such as mean average precision.

\subsubsection{Sampling Distributions}

We test the performance of three sampling distributions:
\begin{enumerate}
    \item Uniform distribution, $q_i \propto 1$, where every class is sampled with the same probability. This provides a convenient baseline.
    \item Softmax distribution, $q_i \propto \exp(o_i)$, which is the ideal sampling distribution as shown in Theorem~\ref{th:unbiased}, but is very expensive to sample from.
    \item Quadratic distribution, $q_i \propto 100(o_i)^2 + 1$, as proposed in Section~\ref{sec:quadratic_kernel}
\end{enumerate}

\begin{figure*}[ht]
    \includegraphics[width=0.66 \columnwidth]{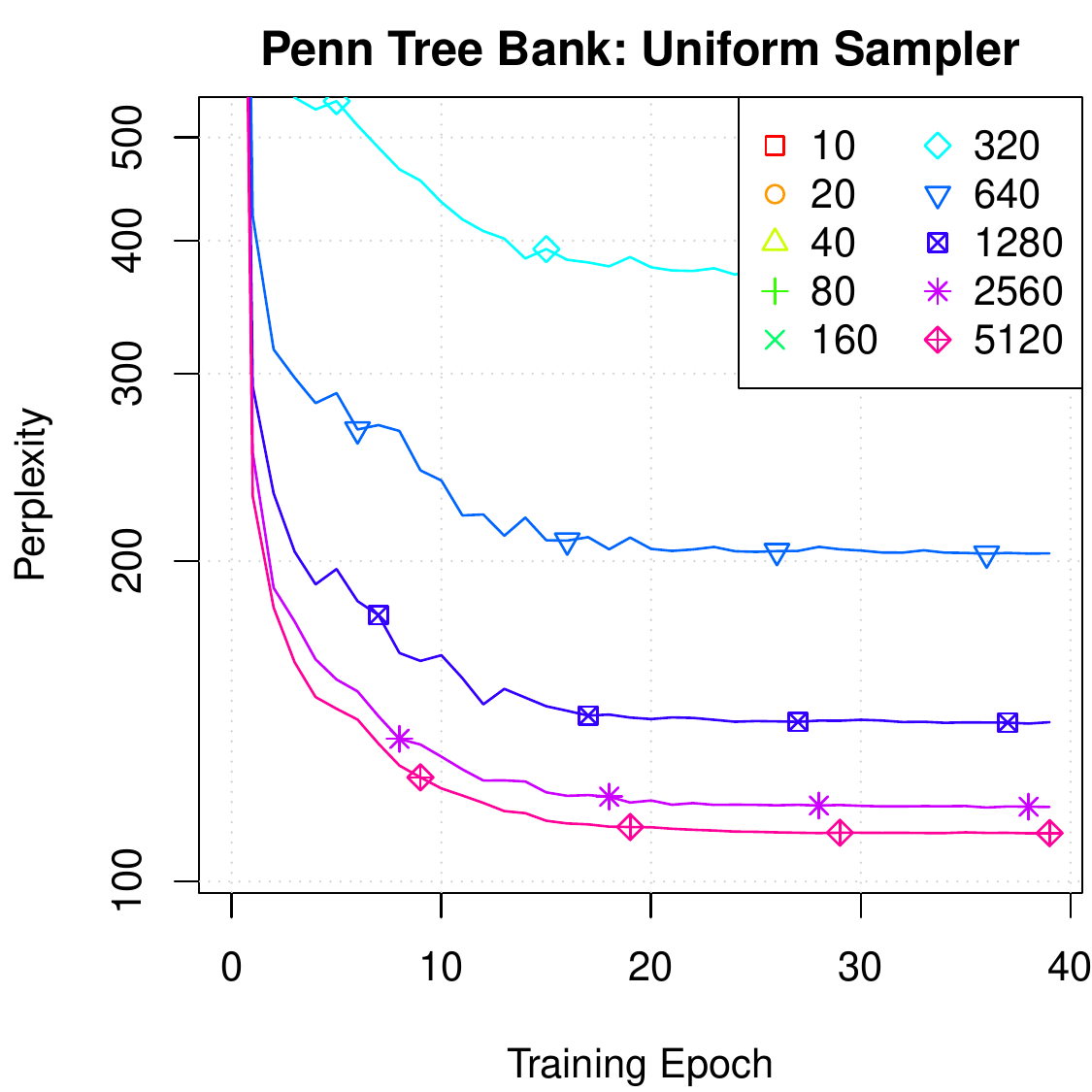}
    \includegraphics[width=0.66 \columnwidth]{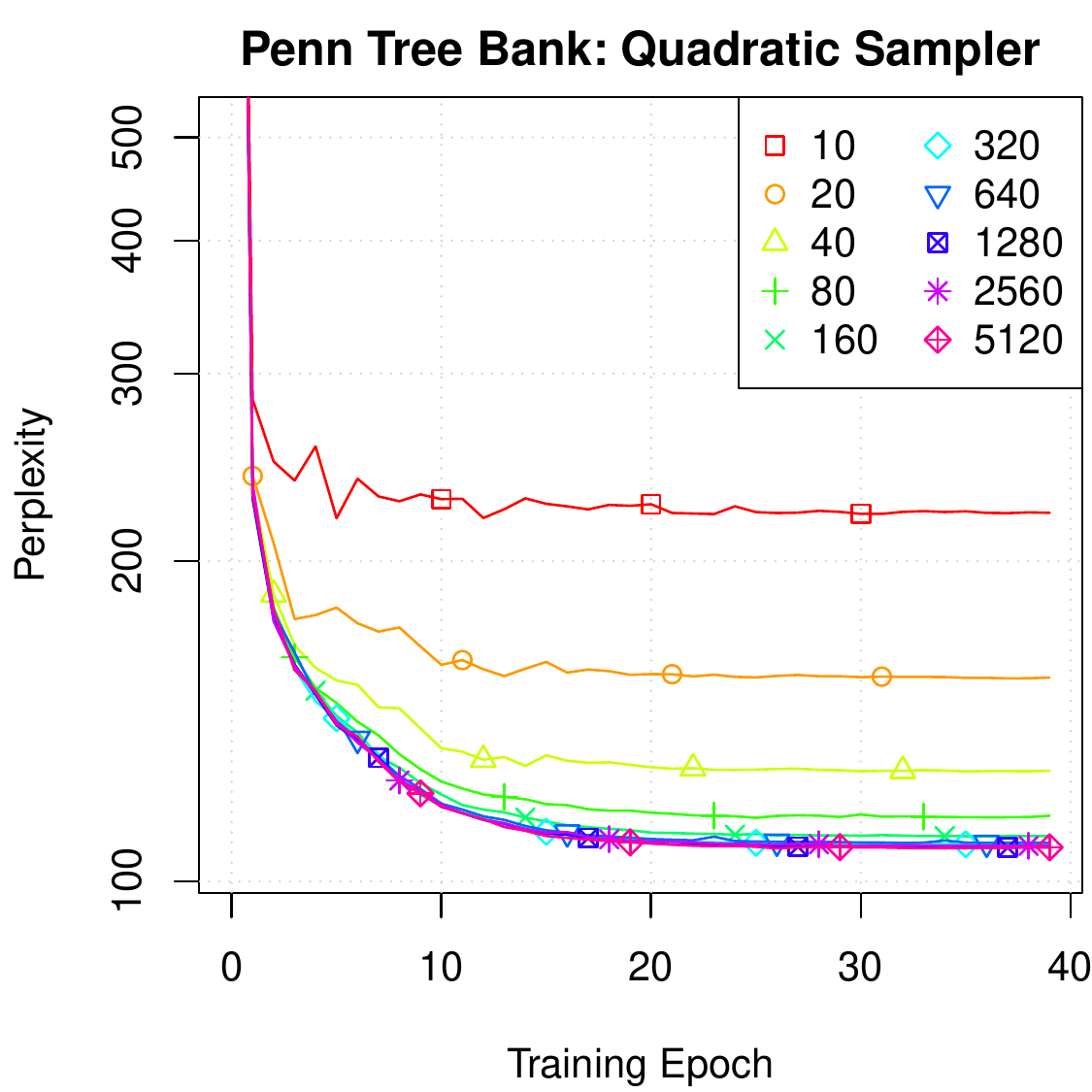}
    \includegraphics[width=0.66 \columnwidth]{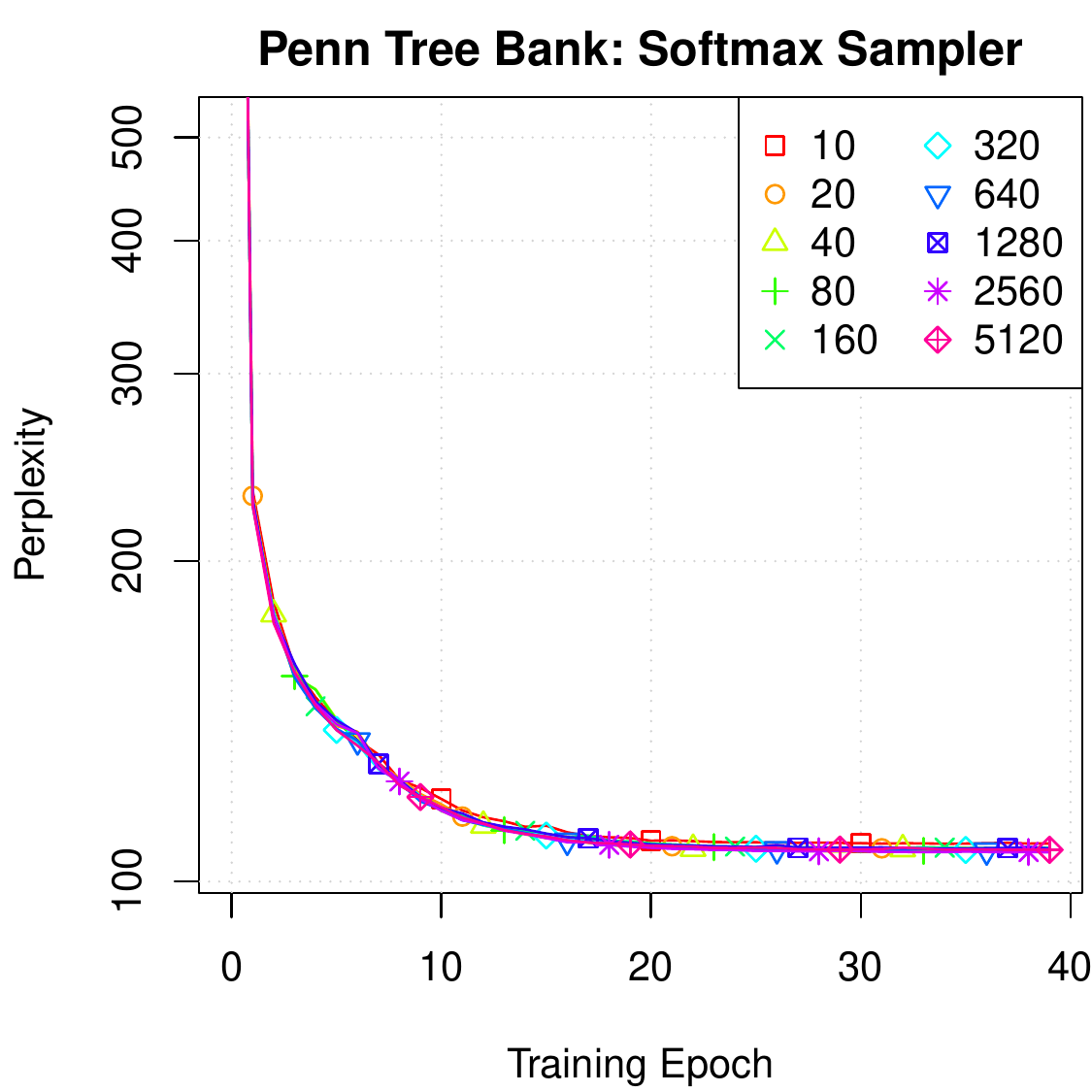}
    \caption{
    Convergence speed for a varying sample size $m \in \{10,20,40,\ldots\}$.
    Once enough samples are taken to remove the bias, adding more samples does not increase convergence speed considerably.
	Additional results for YouTube10k and YouTube100k as well as other samplers for Penn Tree Bank show a similar behavior an can be found in Figures~\ref{fig:exp_num_samples_ptb_appendix}, \ref{fig:exp_num_samples_appendix}. 
    }
    \label{fig:exp_num_samples}
\end{figure*}

\subsection{Results and Analysis}

\subsubsection{Bias of Sampling}

First, we study the bias of sampled softmax empirically.
According to Section~\ref{sec:unbiased}, any sampled softmax is biased unless softmax is chosen as the sampling distribution, and this bias decreases as the sample size, $m$, increases.
We visualize the bias by learning models with different sampling strategies until convergence and reporting the final accuracy.
Very biased models perform poorly even when they are run until convergence. 

The results are shown in Figure~\ref{Bias Plots}.
As expected from theorem~\ref{th:unbiased}, the quality of softmax sampling, i.e., $\bq \propto \exp(\bo)$, is independent of the number of samples $m$.
This verifies that a "good" sampling distribution does not need many samples.
On the other hand, uniform and quadratic sampling are both biased and their final quality improves with increasing sample size, $m$.
Again, it is important to note that training for more epochs does not solve this issue because the loss that sampled softmax optimized is biased when sampling uniformly or according to a quadratic kernel for any fixed size $m$.
On all datasets, quadratic has a much lower bias than uniform sampling and approaches the loss of softmax with 10s to 100s of samples.

\begin{figure*}[ht]
    \includegraphics[width=0.66\columnwidth]{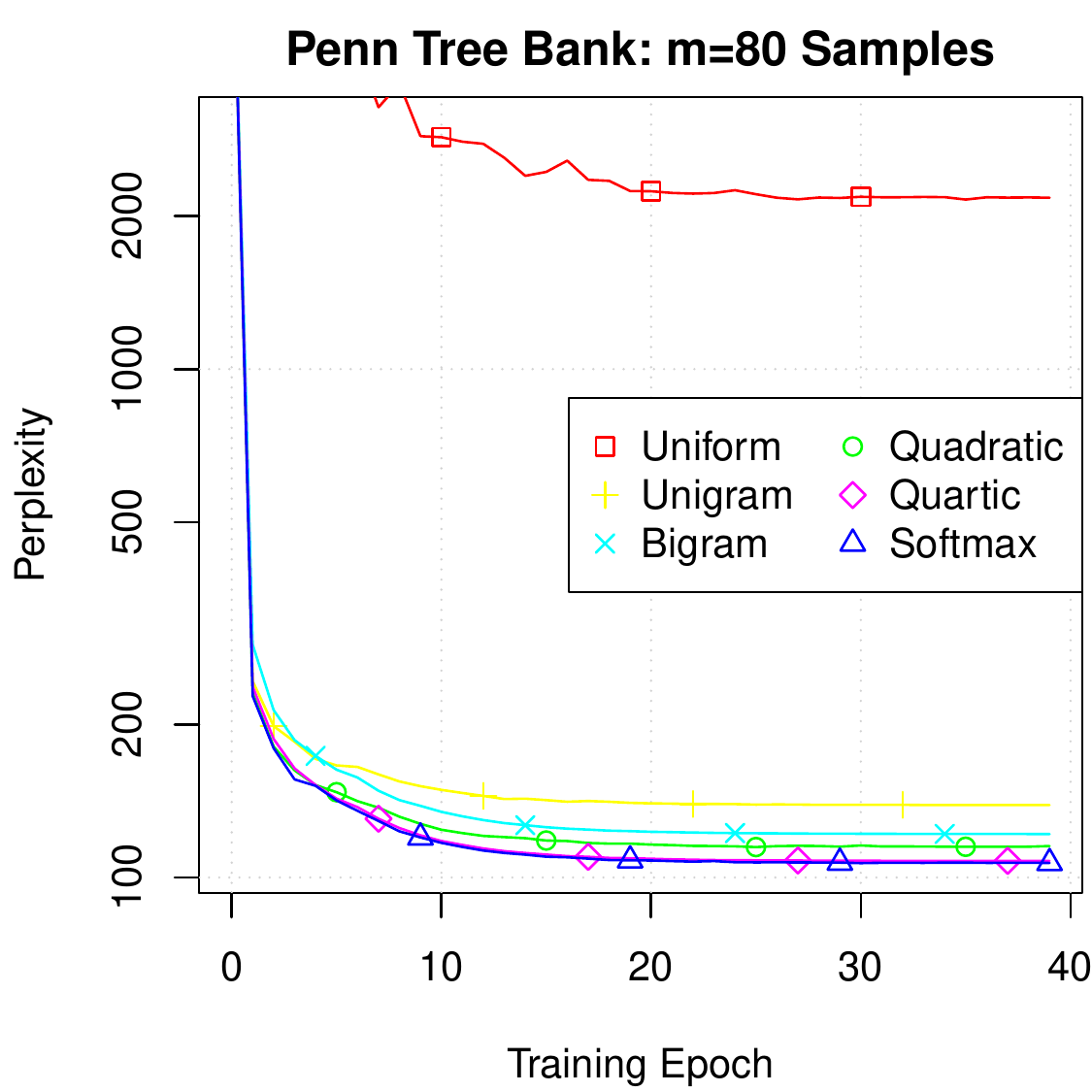}
    \includegraphics[width=0.66\columnwidth]{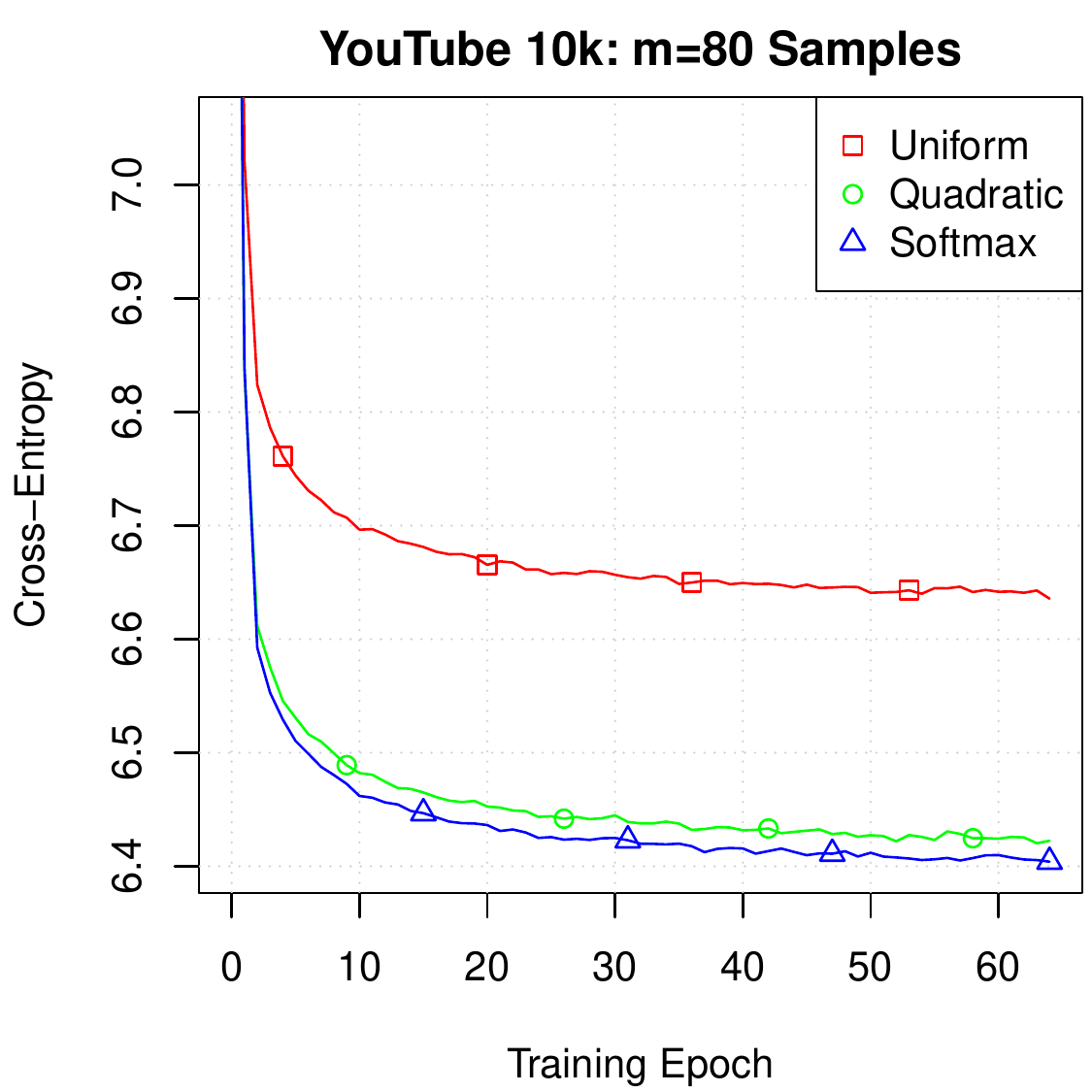}
    \includegraphics[width=0.66\columnwidth]{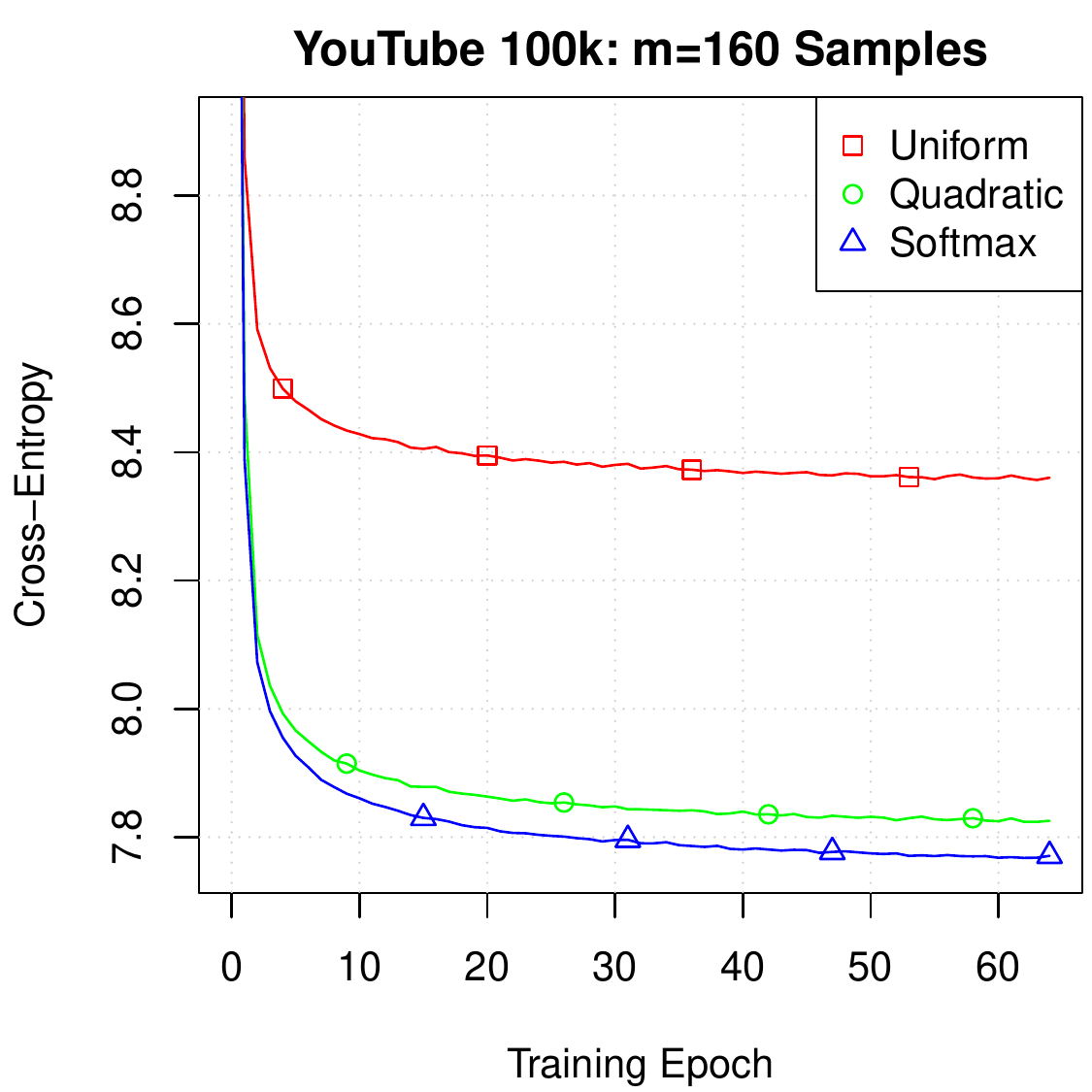}
    \caption{
    Convergence speed of different sampling distributions for a fixed sampling size.
    The convergence speed of all distributions is similar only the bias is different.
    Figure~\ref{fig:exp_distributions_appendix} shows more comparisons.
    }
    \label{sampling plot}
\end{figure*}

\subsubsection{Convergence Speed}

Second, we study the speed of convergence by measuring the progress of the loss against the number of training epochs.
Every update step consists of reading a batch of training examples, sampling $m$ negative classes per example and performing the update with sampled softmax.
We plot loss against epochs instead of wall runtime to eliminate any implementation specific artifacts.
Please note that the larger the sample size $m$, the more computationally expensive an epoch.

\paragraph{Sample Size}

First, we study how the sample size, $m$, influences convergence speed.
Figure \ref{fig:exp_num_samples} shows the convergence for the three sampling strategies.
As already discussed, we see that the number of samples has a large effect on the accuracy of the model for the uniform and quadratic sampler.
Interestingly, once enough samples are taken to remove the bias, adding more samples appears to have a small and mostly unobservable effect on convergence speed.
We previously discussed how bias of the sampled softmax estimator affects the final optimum achieved.
The variance of this estimator affects how many steps we need to converge.
The variance has two sources:
(i) The gradient computed on a batch is a noisy (but unbiased) estimator of the gradient on the entire training set and
(ii) the gradient given a set of sampled classes is an estimator of the gradient on that batch.
While taking a larger sample size can reduce the variance from source (ii), if the variance from source (i) is the dominate source, doing so will not appreciably increase convergence speed.
For our data sets, we found that once we take a reasonable number of samples (only 10s), adding more does not noticeably increase convergence speed.
This is likely because the variance from source (i) dominates that from source (ii).
For instance on Penn Tree Bank, quadratic sampling with $m \in \{160, 320, \ldots\}$ samples does not show any difference in convergence speed.

To summarize, the sample size $m$ influences the bias but the influence on the convergence speed is small and often not noticeable.

\paragraph{Sampling Distribution}

Finally, we fix the number of samples $m$ and vary the sampling distribution.
Figure~\ref{sampling plot} shows that all three sampling distributions have a comparable speed of convergence, however, uniform converges to a much worse loss due to its high bias.
Quadratic and softmax converge similarly although quadratic has a slightly worse loss throughout the whole training process due to its bias.

\section{Related Work}

In this section, we summarize the main approaches for training classification models over many classes.
All of them make some approximation of the full softmax to lower the computational complexity.

\subsection{Sampled Softmax}

Other works on sampled softmax have noted that a good sampling distribution can boost performance and attempted to come up with such distributions.
\citet{bengio:08} propose an adaptive sampler for language models.
They argue that the sampling distribution should track the model distribution as closely as possible.
They propose to learn a mixture of unigrams, bigrams, trigrams, etc. that is adapted while training.
While the work of \citet{bengio:08} needs a second model to track the trained model, our work uses the trained model directly for sampling.
This makes our approach much easier to apply.
Secondly, kernel based sampling is more appealing for sophisticated model structures where it is hard to come up with a simple model that can track the trained model well.
\citet{labeau:17} study sampling distributions for noise contrastive estimation (NCE)~\cite{gutman:10}.
Their experiments highlight the issues of simple sampling distributions such as uniform, or unigram.
Another idea to improve the sampling distribution is the Two-Pass Approximate Adaptive Sampling for Softmax (TAPAS).
In that work, \citet{bai:17} propose taking one large sample of classes, which might be in the order of 100,000 (20\% of all classes in their case) and computing the logits from that sample.
Then, a smaller number of classes, e.g., 1,000, is chosen from those 100,000 classes based on the computed logits.
This second sample of 1,000 classes is used for the sampled softmax.
By using a distributed implementation and GPUs, it is possible to compute the logits of the larger sample quickly.
While the TAPAS sampler is adaptive and depends on the current model's output as in our work, it is computationally much more expensive.
\citet{bakhtiary:15} also explore selectively computing logits using hashing to obtain faster training steps for large batch sizes.

\subsection{Hierarchical Softmax and Its Variations}

Hierachical Softmax (HSM) is an approximation of a full softmax introduced in \cite{googman:01} that is quickly computable.
It involves grouping the classes into clusters, where each cluster is a latent variable.\footnote{Some of the literature refers to what we call classes as simply words and uses classes to refer to what we call clusters. We chose the term \emph{classes} instead of \emph{words} to stress that hierarchical softmax can apply to contexts outside of NLP.}
If $c_j$ is the $j^{\text{th}}$ cluster and class $i$ is in $c_j$, then we factor $p_i$ as $p(y_i|x) = p(c_j | x)\, p(y_i | c_j)$. If we set the number of classes in each cluster to be  $O(\sqrt{n})$ and the cluster probabilities can be computed in time $O(d)$, then this version of hierarchical softmax can be done in $O(d\sqrt{n})$.

\citet{morin:05} extend this structure to a tree.
Instead of having one layer of clusters they use a binary tree where each internal node is a cluster and the leaf nodes are the classes.
The probability of a class is then the product of the conditional probability of each node along the path from the root to that class. Such a structure allows for $O(d \log n)$ training time.

While hierarchical softmax can be much faster than a full softmax, it often performs worse at convergence.
For instance, \citet{chen:15} found full softmax to achieve a perplexity more than 10\% better than hierarchical softmax.
They also note that while hierarchical softmax can speed up training, it slows down inference if the goal is to compute the class or classes with the highest logits.
In particular, both a full softmax and sampled softmax can treat inference as a maximum inner product search, which can be done in sublinear time with methods such as locally sensitive hashing \cite{shrivastava:14} or clustering \cite{auvolat:15}.
The tree structure has a large effect both on the final performance and the efficiency of each training step, so there is much work on modifying that structure.
Various approaches have been used to build this tree, such as by class similarity \cite{le:11}, by frequency binning \cite{mikolov:11}, or to optimize the speed of the model \cite{grave:17}.
See \citet{zweig:13} for experimental results showing the effects of some common tree structures.

\subsection{Spherical Softmax and Kernels}

\citet{vincent:nips15} propose to optimize a variation of softmax referred to as spherical softmax.
In spherical softmax, the prediction distribution is changed by replacing the $\exp$ in eq.~(\ref{softmax probability}) with a quadratic function.
This alternative formulation allows exact gradient computations without computing all the logits, needing only $\O(d^2)$ time.
This time is thus independent of the number of classes, resulting in a significant speedup.
\citet{spherical:15} note that on some problems, spherical softmax produces models with comparable quality as full exponential softmax (eq.~\ref{softmax probability}), but on other problems the quality is considerably worse.
Especially for problems with many classes, optimizing spherical softmax seems to produce low quality results.
We also found the spherical formulation not to be as effective as an exponential softmax on our datasets.
Finally, our approach of kernel based sampled softmax with a quadratic kernel can be viewed as using the spherical softmax for sampling, and then the normal exponential softmax formulation for computation of the loss.
As in our approach the quadratic kernel is only used for sampling, the high quality of exponential softmax is preserved.
\citet{rastogi:15} propose kernel feature maps to approximate the softmax partition function during inference after a model has been learned.
In contrast to this, our work focuses on using kernels during training which requires dealing with parameter updates.

\section{Conclusion}

This work shows the importance of the sampling distribution when learning a sampled softmax model.
In particular, any sampling distribution besides softmax is biased and converges to a worse quality than full softmax -- no matter how many learning steps are taken.
The only way to mitigate the bias is to increase the sample size or to use a better sampling distribution.
A good sampling distribution should depend on the model output, which requires the distribution to be example dependent, model structure dependent and model parameter dependent.
We have introduced the new class of kernel based sampling methods that sample based on the model output.
Kernels allow efficient sampling even for a large number of classes as they depend only on the dimension of the kernel space and with a divide and conquer algorithm on the logarithm of the number of classes.
On several experiments, a quadratic kernel showed a one to two order better sample efficiency than uniform sampling.

Both the sampling algorithm as well as the kernel function offer several directions for future work.
Besides other analytical kernels such as polynomial kernels, random feature maps~\cite{rahimi:08} could provide another rich class for constructing kernels.
For particular classes of kernels, more efficient sampling algorithms might exist, e.g., for non-negative maps, $\bphi$, a clever use of Alias sampling~\cite{walker:77} could provide an $\O(D)$ time sampling algorithm.

\bibliography{paper}
\bibliographystyle{icml2018}

\section*{Appendix}

\subsection*{Sampled Softmax is biased}

\begin{theorem}
\label{th:unbiased}
The gradient of sample softmax is an unbiased estimator of the full softmax gradient iff $q_i = p_i \propto \exp(o_i)$.
\end{theorem}

\begin{proof}
Note that the random variable in eq.~(\ref{eq:exgradient}) is the sample $\bs$.
Remember that the sample consists of one positive label that is chosen with probability $1$ and $m$ negatives that are sampled from $\bq$.
Our proof analyzes both parts separately.
For notational convenience and without loss of generality, we assume that the positive class has index 1, $y_1 = 1$, and that the positive is at the first position in the sample, i.e., $s_1 = 1$.
 
Case 1: First we analyze the output $o_i$ of the positive label, i.e., $y_i=1$.
Under our notational assumptions $i=1$.
\begin{multline*}
       E\left[\sum_{j=1}^{m+1} I(s_j = i) p'_j\right] = E[p'_1] \\
       = E\left[\frac{\exp(o_1)}{\exp(o_1) + \sum_{k=2}^{m+1} \exp(o'_k)}\right]\\
     \stackrel{(A)}{\geq}  \frac{\exp(o_1)}{\exp(o_1) + E\left[\sum_{k=2}^{m+1} \exp(o'_k)\right]} \\
     \stackrel{eq. (\ref{eq:helper})}{=}  \frac{\exp(o_1)}{\exp(o_1) + \sum_{k=2}^{n} \exp(o_k)} = p_i
\end{multline*}
The last step is based on:
\begin{multline}
     E\left[\sum_{k=2}^{m+1} \exp(o'_k)\right] = \sum_{k=2}^{m+1} E\left[\frac{\exp(o_{s_k})}{m q_{s_k}}\right] \\
   =  \sum_{k=2}^{m+1} \sum_{j=2}^{n} q_j \frac{\exp(o_j)}{m q_j} = \sum_{j=2}^{n} \exp(o_j) \label{eq:helper}
\end{multline}
Step $(A)$ is Jensen's inequality.
This turns into an equality if and only if $\sum_{k=2}^{m+1} \exp(o'_{s_k})$ is constant which means $\exp(o'_l)=\exp(o_{s_l} -\ln m q_{s_l})=\frac{\exp(o_{s_l})}{m q_{s_l}}$ has to be constant.
This is true iff $q_{s_l} \propto \exp(o_{s_l})$.
In other words, for a positive label, the softmax distribution is the only choice to create an unbiased sampled softmax.

Case 2: Let $i$ be a negative class such that $y_i=0$.
Under our notational assumptions, $i>1$.
For the second case, we only show that softmax is unbiased, i.e., the sufficient condition of the theorem.
We don't show that softmax is the only distribution that is unbiased.
This necessary condition is already covered by case 1.
Let the negative sampling distribution be the softmax $q_i := \frac{\exp(o_i)}{\sum_{l=2}^n \exp(o_j)}$.
\begin{multline*}
       E\left[\sum_{j=1}^{m+1} I(s_j = i) p'_j\right]\\
     = E\left[\sum_{j=2}^{m+1} I(s_j = i)\frac{\exp(o'_j)}{\exp(o_1) + \sum_{k=2}^{m+1}\exp(o'_k)}\right]\\
    \stackrel{eq.~(\ref{eq:helper2})}{=} \frac{1}{m} E\left[\sum_{j=2}^{m+1} I(s_j = i)\frac{1}{q_i}\frac{\exp(o_i)}{\exp(o_1) + \sum_{k=2}^{n}\exp(o_k)}\right] \\
    = \frac{1}{m} E\left[\sum_{j=2}^{m+1} I(s_j = i)\frac{p_i}{q_i}\right]= p_i
\end{multline*}
This proof uses that for the softmax sampling distribution, the sum of corrected sampled outputs is equal to the sum of all outputs.
\begin{align}
      &\sum_{k=2}^{m+1}\exp(o'_k) = \sum_{k=2}^{m+1}\frac{\exp(o_{s_k})}{m\,q_{s_k}} \label{eq:helper2}\\
    = &\sum_{k=2}^{m+1}\frac{\exp(o_{s_k})}{m}\frac{\sum_{l=2}^{n}\exp(o_l)}{\exp(o_{s_k})} = \sum_{l=2}^{n}\exp(o_l) \notag
\end{align}
This equality is related to eq.~(\ref{eq:helper}).
While eq.~(\ref{eq:helper}) holds for any sampling distribution but only in expectation, eq.~(\ref{eq:helper2}) holds only for softmax but for any sample.
\end{proof}

\begin{figure*}[ht]
    \includegraphics[width=0.66 \columnwidth]{fig/ptb_uniform.pdf}
    \includegraphics[width=0.66 \columnwidth]{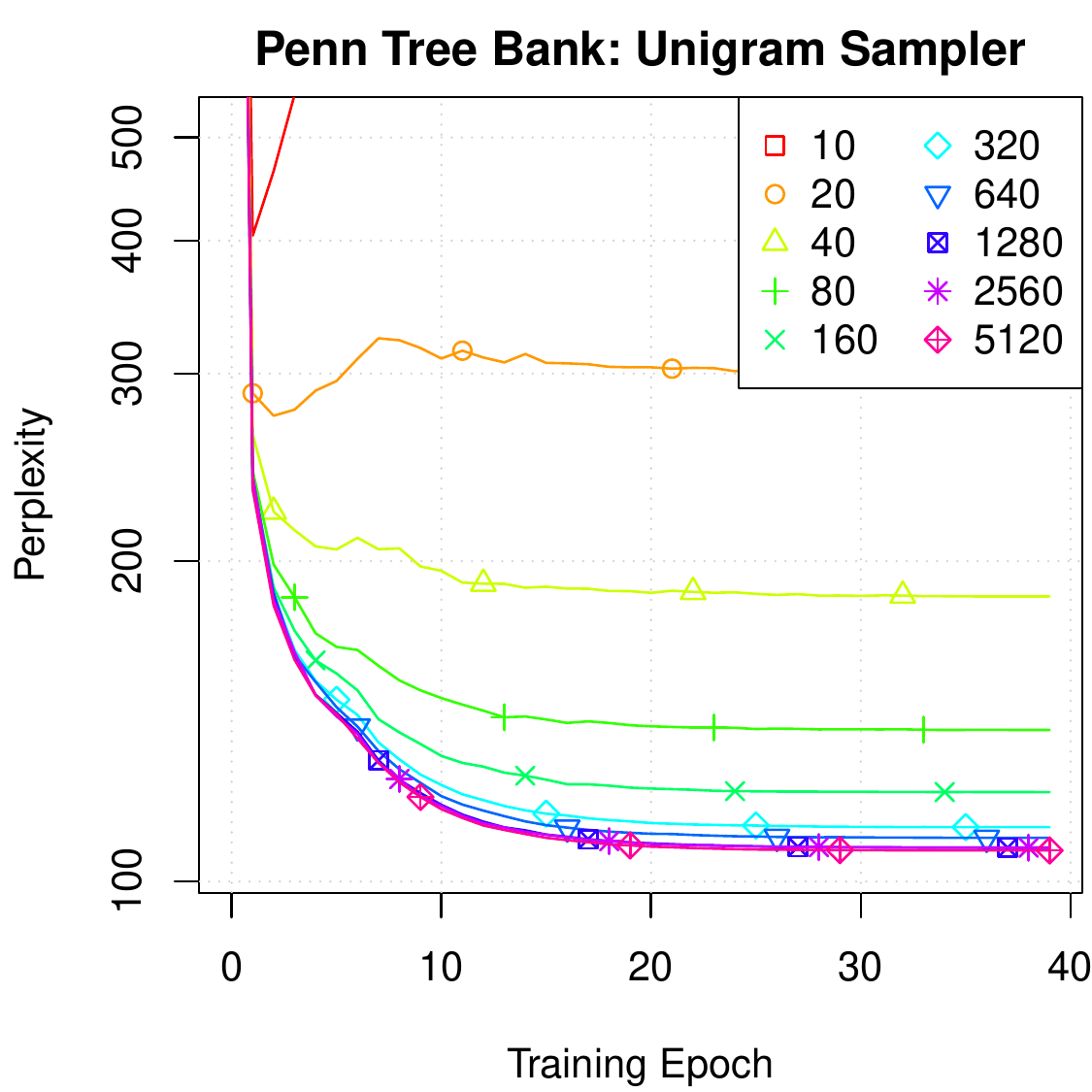}
    \includegraphics[width=0.66 \columnwidth]{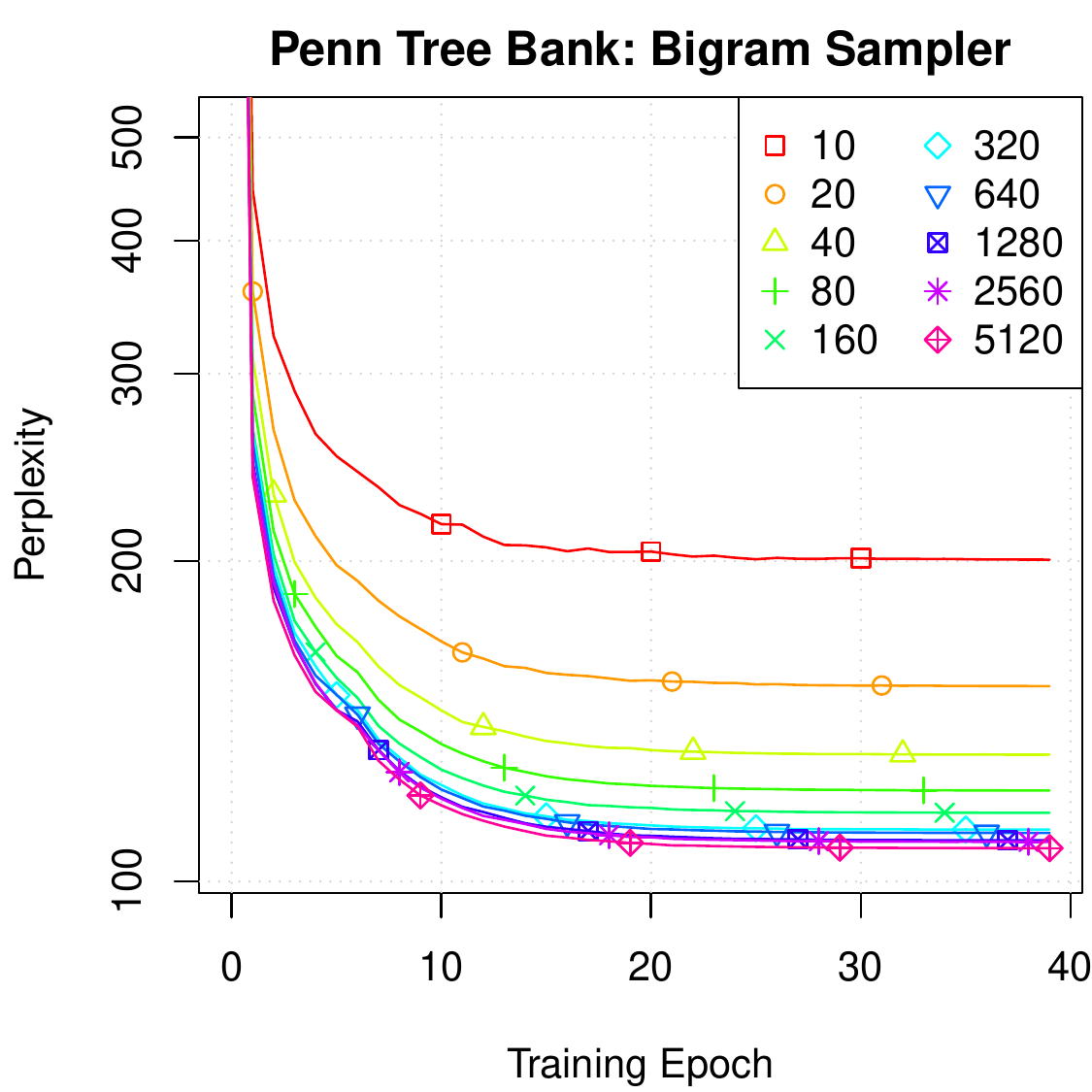}\\
    \includegraphics[width=0.66 \columnwidth]{fig/ptb_quadratic.pdf}
    \includegraphics[width=0.66 \columnwidth]{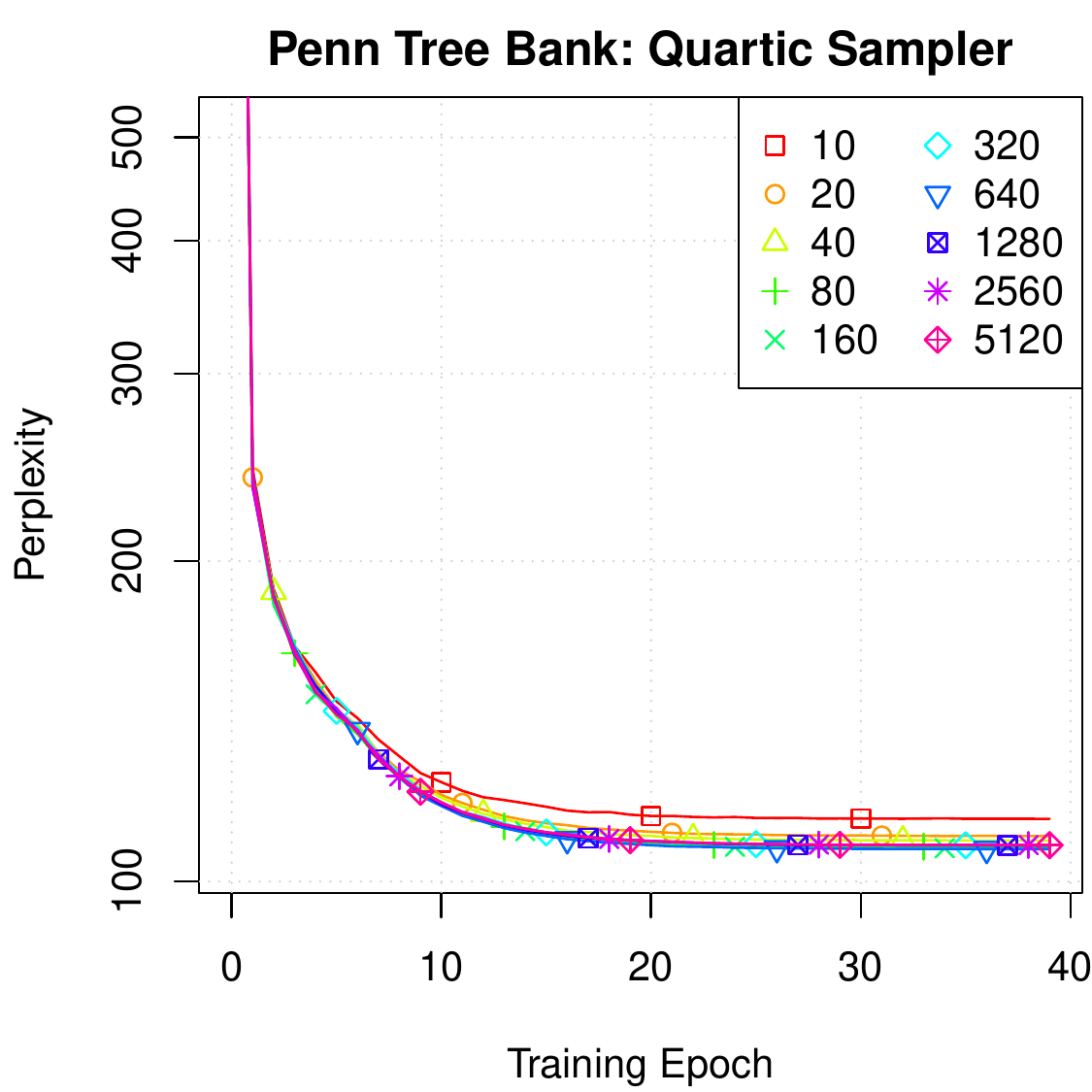}
    \includegraphics[width=0.66 \columnwidth]{fig/ptb_softmax.pdf}
    \caption{
	    Convergence speed of different sampling distributions (\emph{uniform}, \emph{unigram}, \emph{bigram}, \emph{quadratic}, \emph{quartic}, \emph{softmax}) for a varying sample size $m \in \{10,20,40,\ldots\}$ on the \emph{Penn Tree Bank} dataset.
	Once enough samples are taken to remove the bias, adding more samples does not increase convergence speed considerably.
    }
    \label{fig:exp_num_samples_ptb_appendix}
\end{figure*}

\begin{figure*}[ht]
    \includegraphics[width=0.66 \columnwidth]{fig/ptb_uniform.pdf}
    \includegraphics[width=0.66 \columnwidth]{fig/ptb_quadratic.pdf}
    \includegraphics[width=0.66 \columnwidth]{fig/ptb_softmax.pdf}\\
    \includegraphics[width=0.66 \columnwidth]{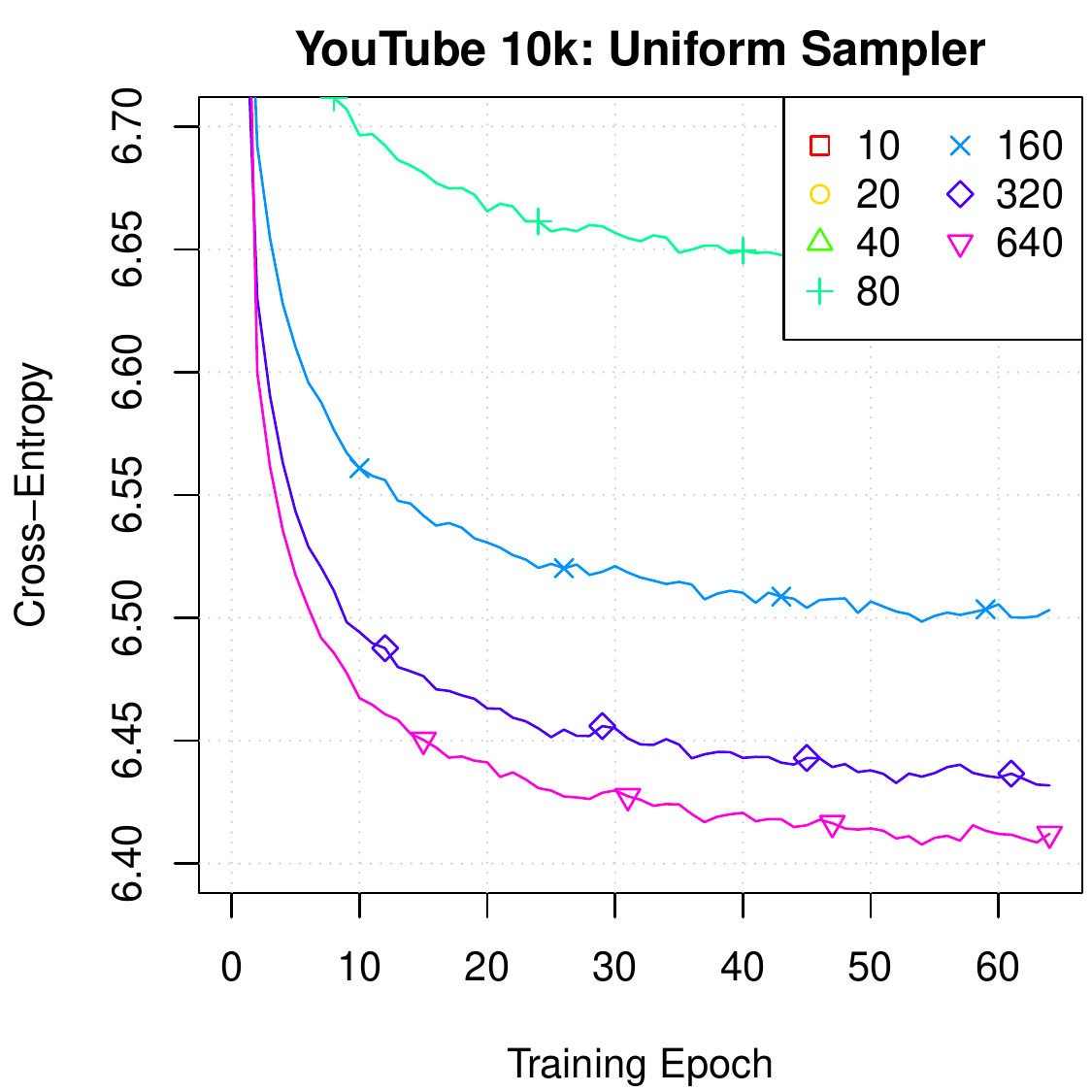}
    \includegraphics[width=0.66 \columnwidth]{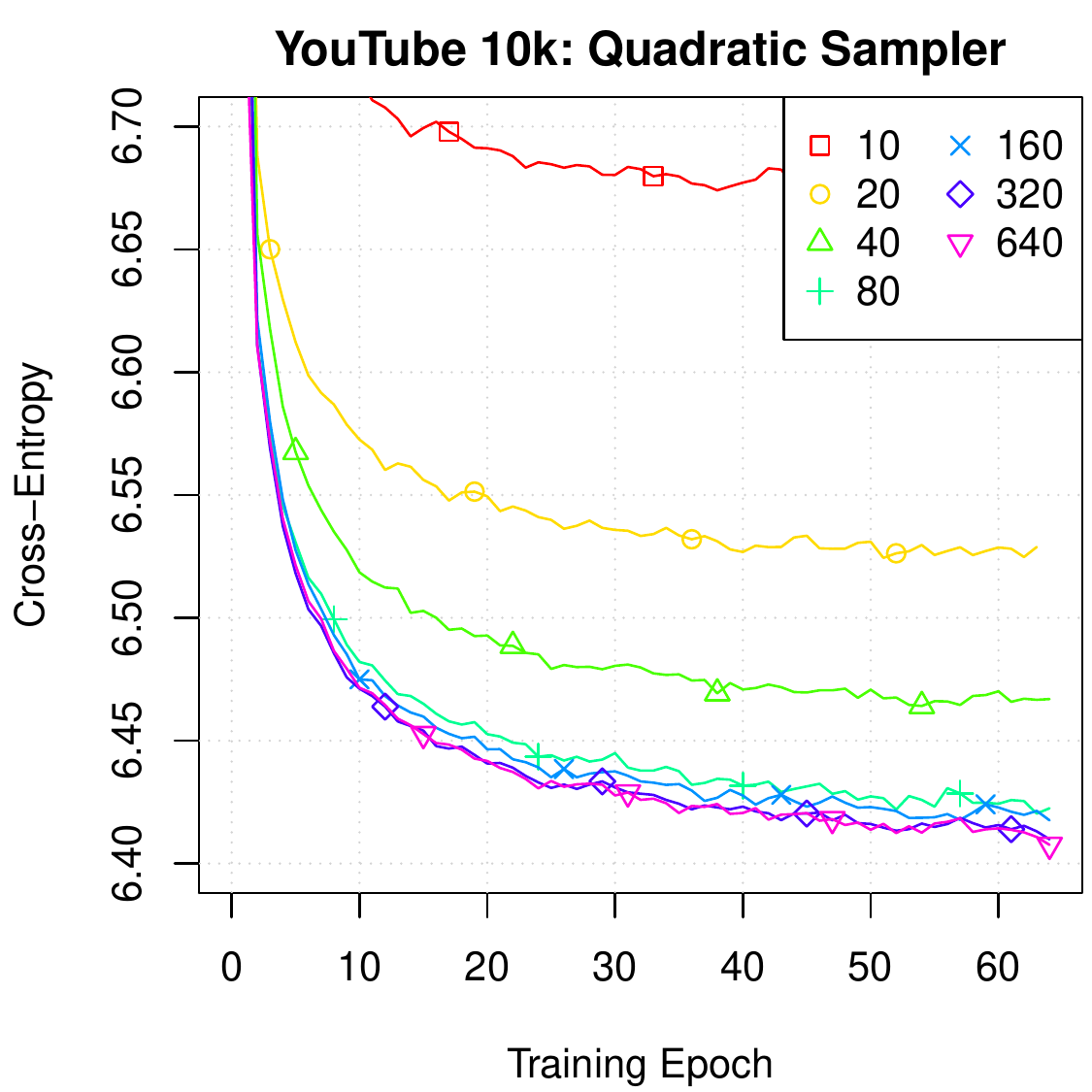}
    \includegraphics[width=0.66 \columnwidth]{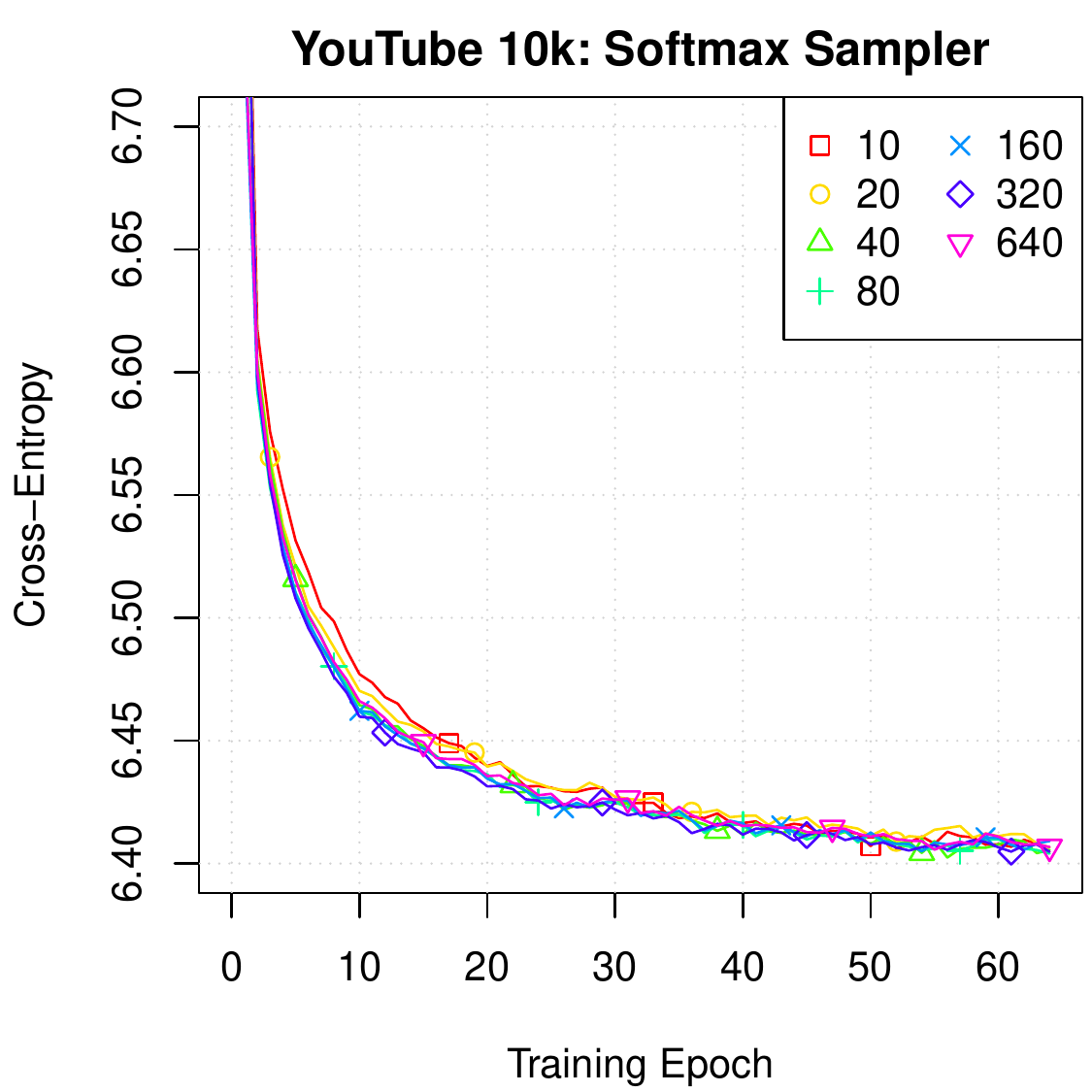}\\
    \includegraphics[width=0.66 \columnwidth]{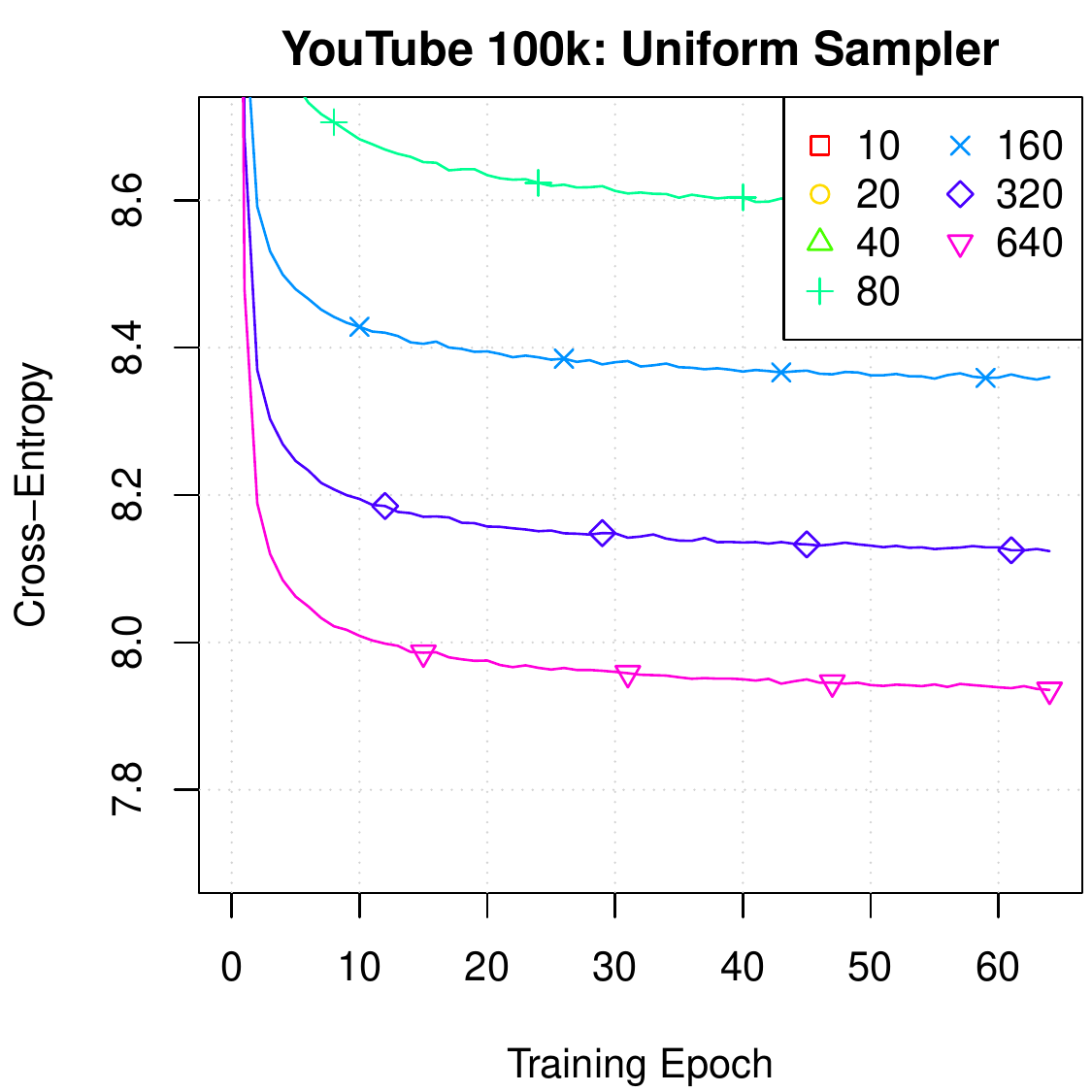}
    \includegraphics[width=0.66 \columnwidth]{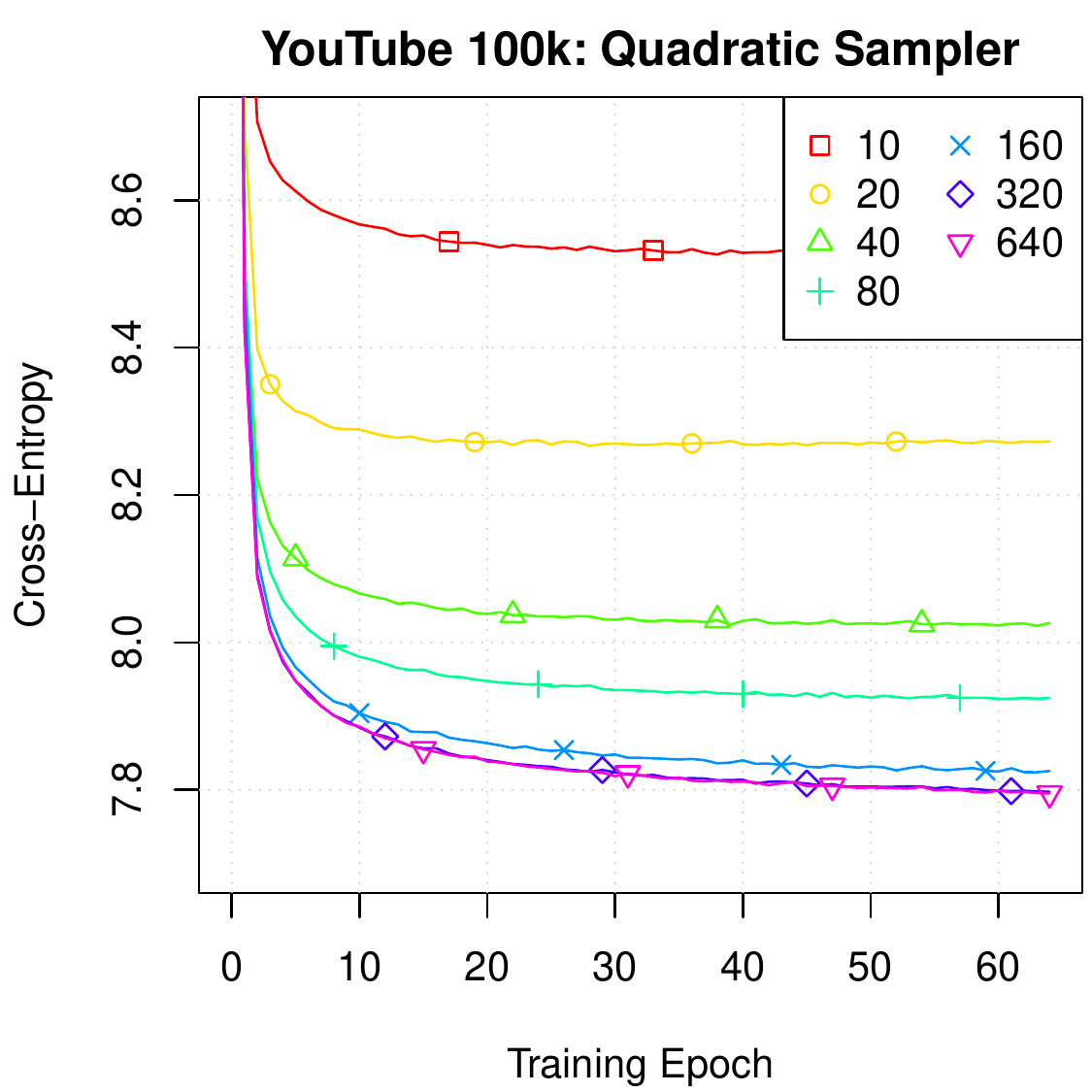}
    \includegraphics[width=0.66 \columnwidth]{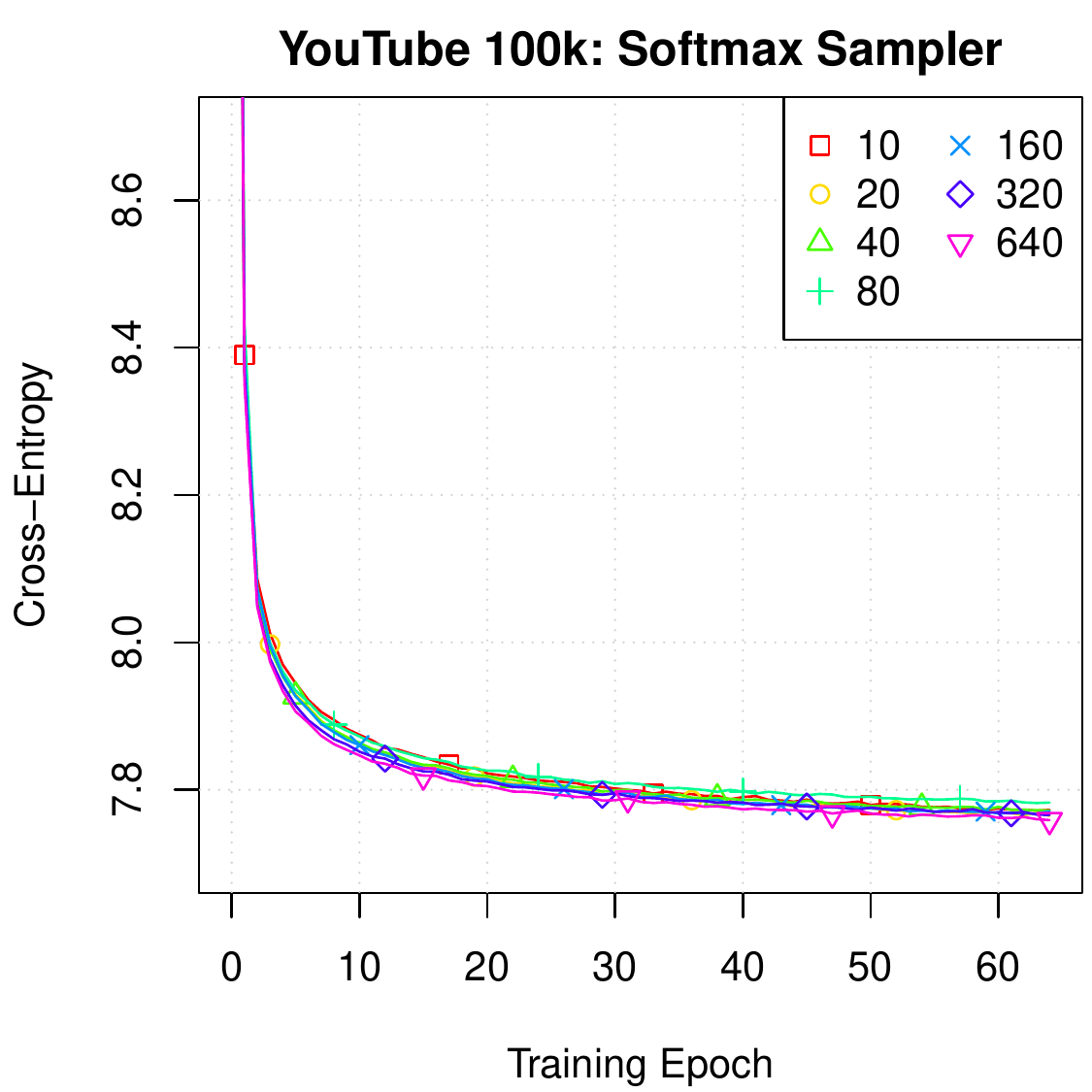}
    \caption{
	    Convergence speed for three sampling distributions (\emph{uniform}, \emph{quadratic}, \emph{softmax}) for a varying sample size $m \in \{10,20,40,\ldots\}$ on three datasets.
    Once enough samples are taken to remove the bias, adding more samples does not increase convergence speed considerably.
    }
    \label{fig:exp_num_samples_appendix}
\end{figure*}

\begin{figure*}[ht]
    \includegraphics[width=0.66\columnwidth]{fig/ptb_samples_80.pdf}
    \includegraphics[width=0.66\columnwidth]{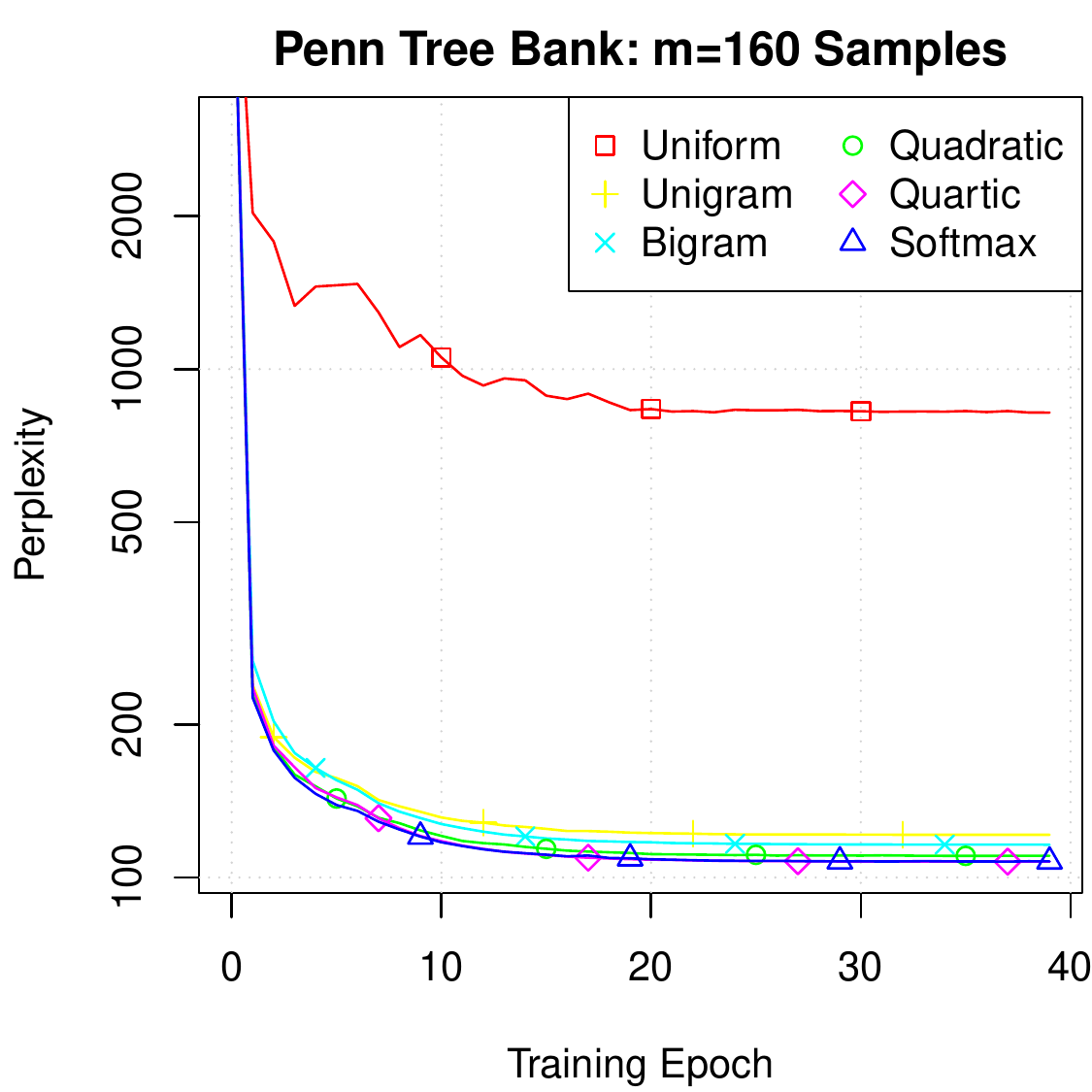}
    \includegraphics[width=0.66\columnwidth]{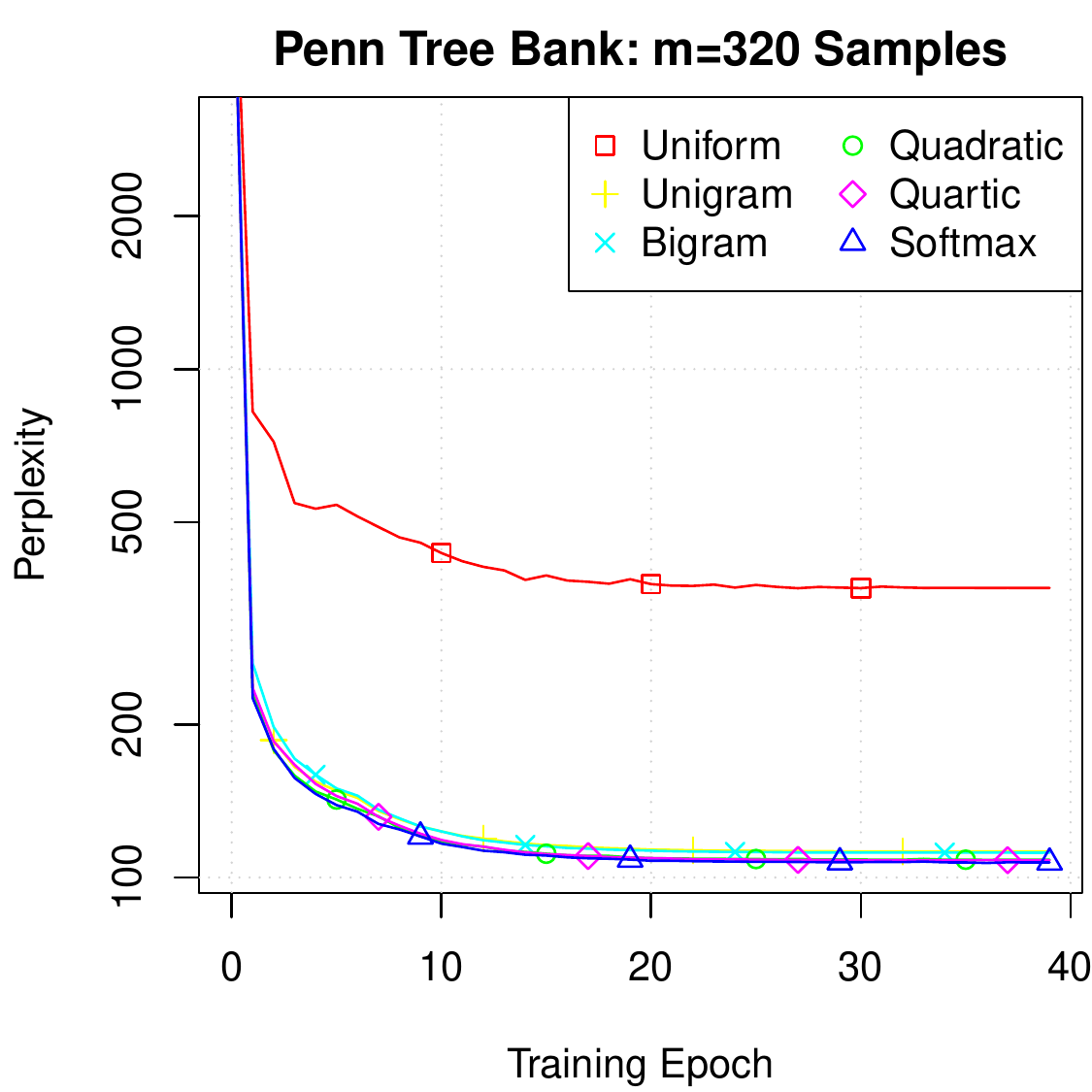}\\
  
    \includegraphics[width=0.66\columnwidth]{fig/yt10k_samples_80.pdf}
    \includegraphics[width=0.66\columnwidth]{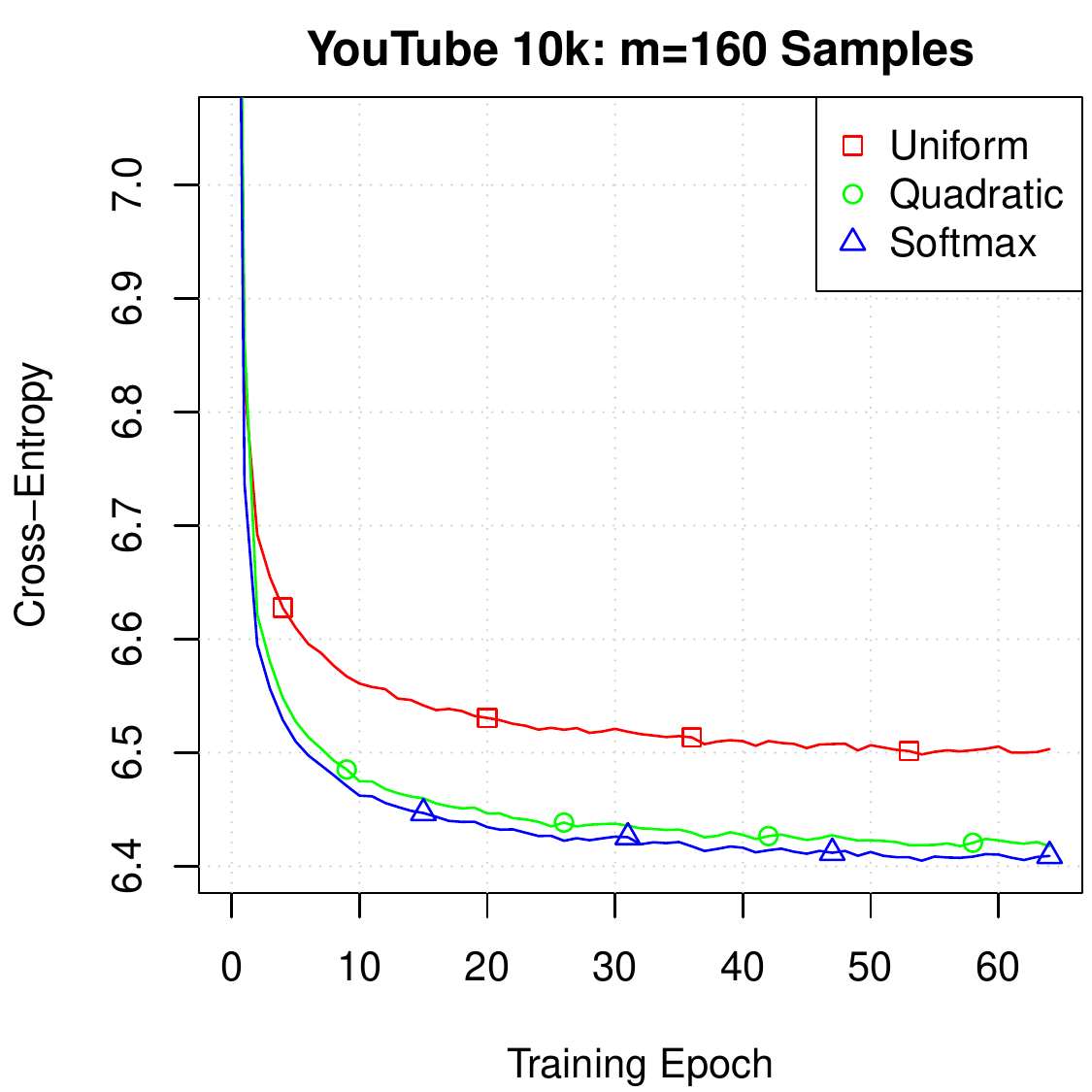}
    \includegraphics[width=0.66\columnwidth]{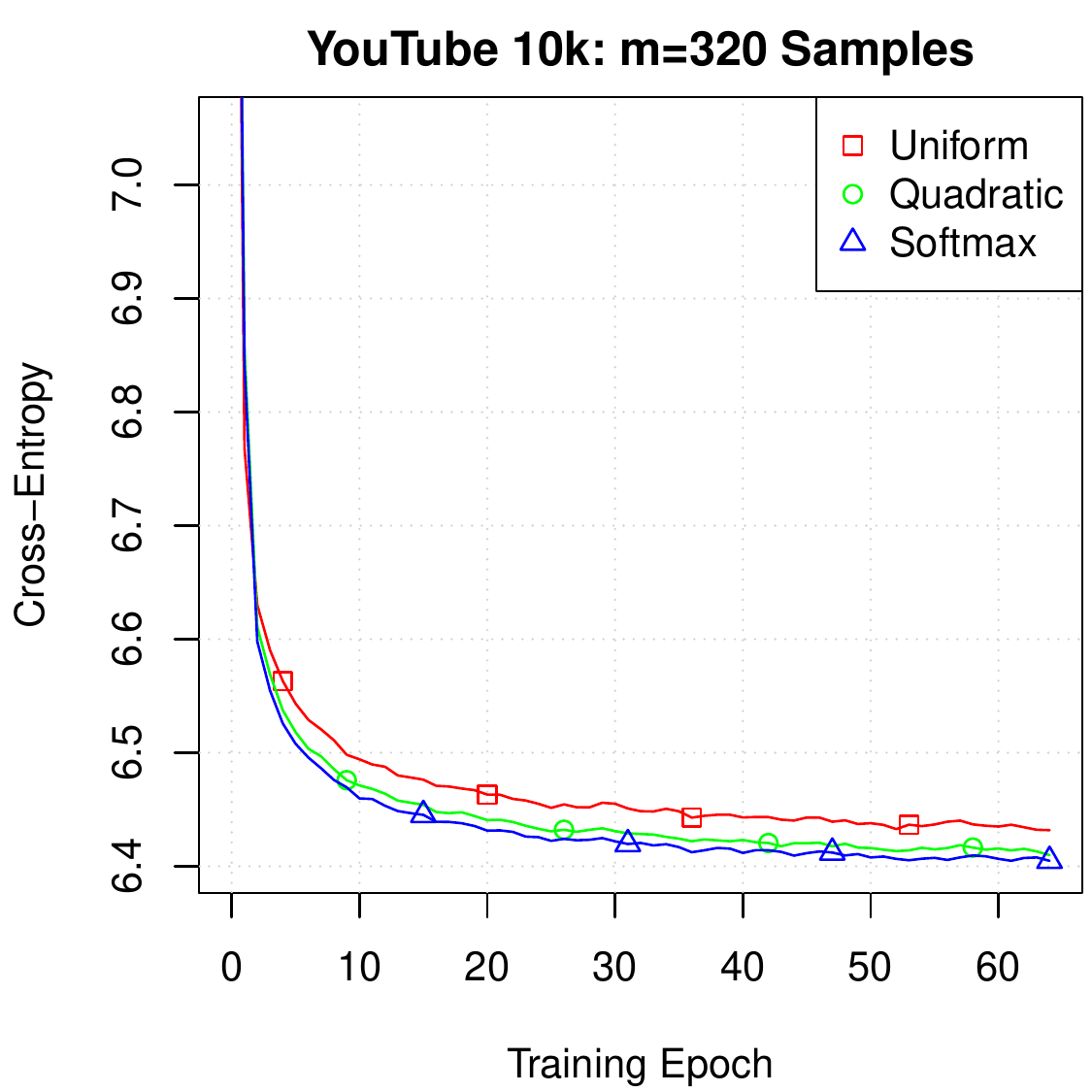}\\

    \includegraphics[width=0.66\columnwidth]{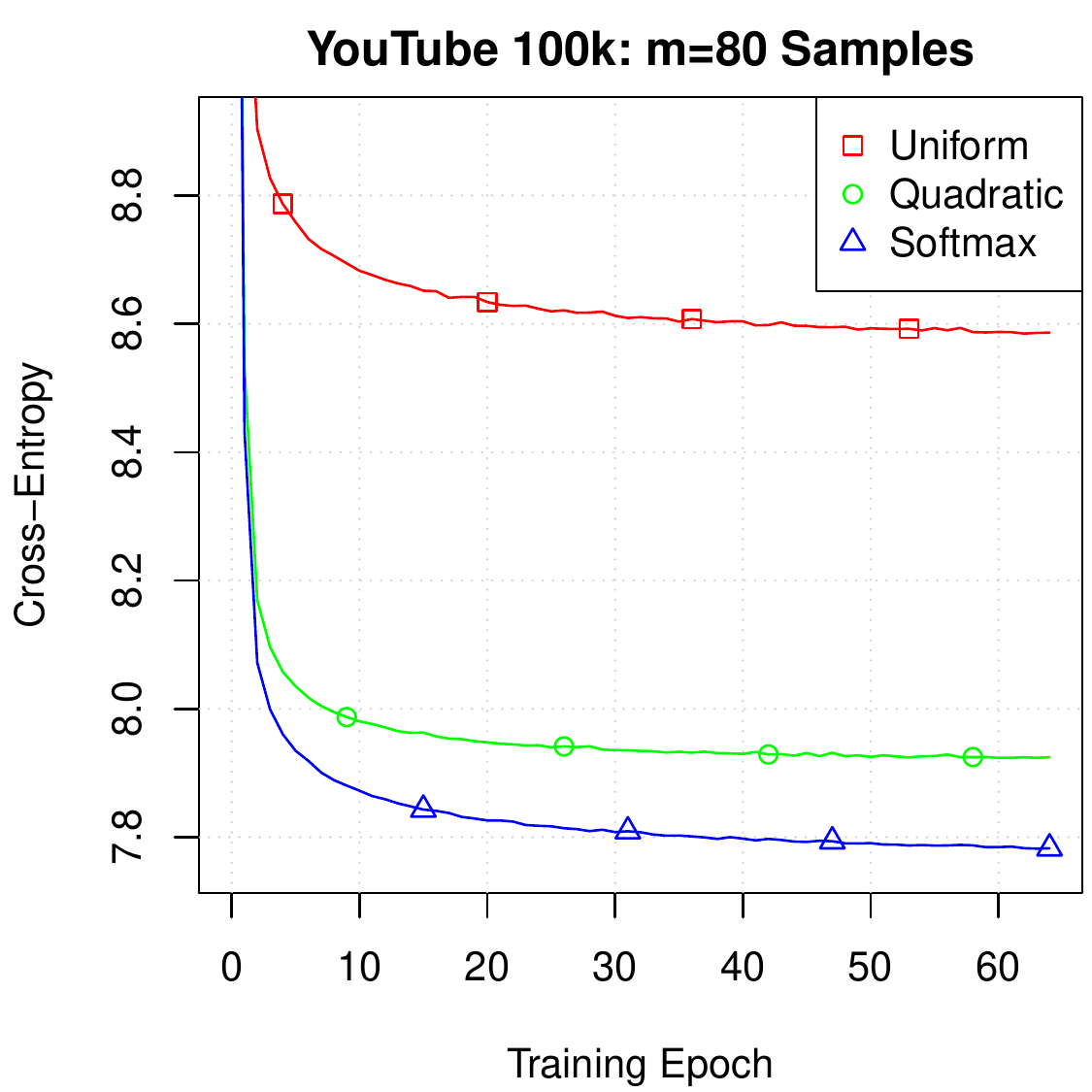}
    \includegraphics[width=0.66\columnwidth]{fig/yt100k_samples_160.pdf}
    \includegraphics[width=0.66\columnwidth]{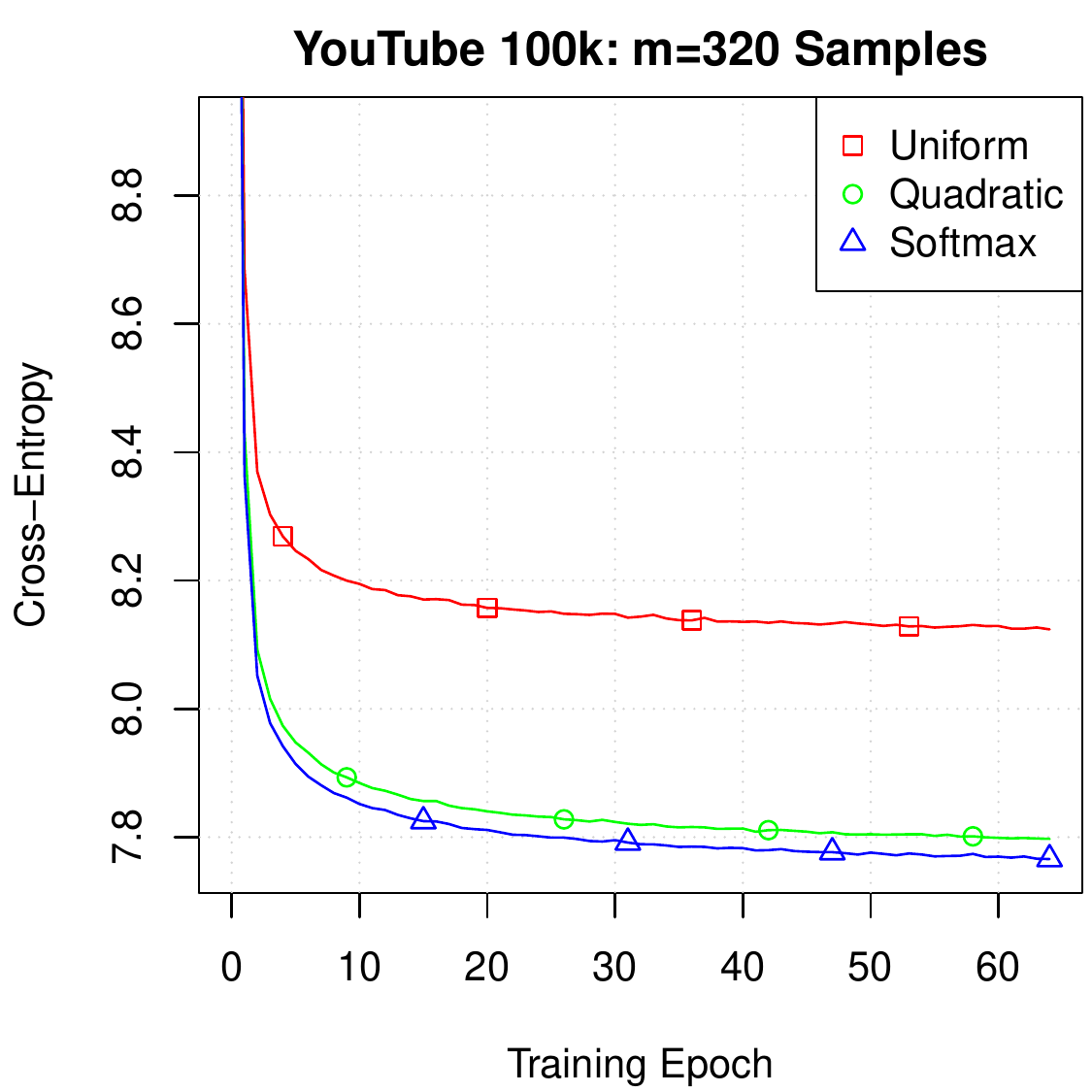}

    \caption{
    Convergence speed of different sampling distributions for a fixed sampling size.
    The convergence speed of all distributions is similar only the bias is different.
    }
    \label{fig:exp_distributions_appendix}
\end{figure*}

\end{document}